\documentclass{article} 
\usepackage{iclr2027_conference,times}

\usepackage{amsmath,amsfonts,bm}

\def\eqref#1{equation~\ref{#1}}

\def\1{\bm{1}}

\DeclareMathAlphabet{\mathsfit}{\encodingdefault}{\sfdefault}{m}{sl}
\SetMathAlphabet{\mathsfit}{bold}{\encodingdefault}{\sfdefault}{bx}{n}

\usepackage[utf8]{inputenc} 
\usepackage[T1]{fontenc}    
\usepackage{hyperref}       
\usepackage{url}            
\usepackage{booktabs}       
\usepackage{xspace}         
\usepackage{nicefrac}       
\usepackage{microtype}      
\usepackage{wrapfig}        
\usepackage{xcolor}
\usepackage{tikz}
\usetikzlibrary{arrows.meta}
\usepackage[most]{tcolorbox}
\newtcbox{\subtitlebox}{on line, 
  arc=3pt, 
  colback=gray!5, 
  colframe=gray!40, 
  before upper=\strut, 
  boxrule=0.6pt, 
  boxsep=0pt, 
  left=4pt, right=4pt, top=2pt, bottom=2pt,
  fontupper=\small \sffamily\bfseries
}

\newtcbox{\subtitleboxblack}{on line, 
  arc=3pt,                
  colback=black,          
  colframe=black,         
  colupper=white,         
  boxrule=0.6pt, 
  boxsep=0pt, 
  left=4pt, right=4pt, top=2pt, bottom=2pt,
  fontupper=\small\sffamily\bfseries 
}

\definecolor{okStroke}{HTML}{15803D}
\definecolor{okFill}{HTML}{D1FAE5}
\definecolor{okText}{HTML}{14532D}

\definecolor{badStroke}{HTML}{B91C1C}
\definecolor{badFill}{HTML}{FECACA}
\definecolor{badText}{HTML}{7F1D1D}

\definecolor{deltapos}{HTML}{15803D}
\definecolor{deltaneg}{HTML}{B91C1C}
\definecolor{deltanil}{HTML}{8A8A8A}
\newcommand{\dlt}[1]{{\small
  \ifdim #1pt>1pt \textcolor{deltapos}{($+$#1)}%
  \else\ifdim #1pt<-1pt \textcolor{deltaneg}{($#1$)}%
  \else \textcolor{deltanil}{(\ifdim #1pt<0pt $#1$\else $+$#1\fi)}%
  \fi\fi}}

\usepackage{graphicx}
\usepackage{subcaption}     

\usepackage{amsfonts, amsmath, amssymb, amsthm}
\usepackage{cleveref}
\usepackage{mathtools}
\usepackage{float}
\newcommand{\singlequery}{single-query\xspace}
\newcommand{\method}{RSG\xspace}

\theoremstyle{plain}
\newtheorem{theorem}{Theorem}[section]
\newtheorem{proposition}[theorem]{Proposition}

\newtheorem{corollary}[theorem]{Corollary}
\theoremstyle{definition}

\newtheorem{assumption}[theorem]{Assumption}
\newtheorem{remark}[theorem]{Remark}

\title{Anchored or Drifting: What Recursive \\Self-Generation Reveals About Training Data}

\author{Wojciech Łapacz\thanks{Equal contribution} \\
Warsaw University of Technology\\
\And
Stanisław Pawlak\footnotemark[1] \\
Warsaw University of Technology \\
\texttt{stanislaw.pawlak.dokt@pw.edu.pl} \\
}

\iclrfinalcopy 
\begin{document}

\maketitle

\begin{abstract}
    Large generative models are known to memorize their training data, posing severe privacy risks. Yet, current methods to detect training membership typically rely on the weak signals of a single forward pass. In this work, we find that training samples and unseen (held-out) data follow visibly different trajectories under recursive self-generation -- repeatedly feeding a model's output back as its next input. Held-out samples \emph{drift}: they lose the specifics of the original within a few steps. Training samples stay \emph{anchored}, degrading far more slowly. We show that the membership signal this produces holds across model modalities, architectures and scales, spanning language, diffusion, and autoregressive vision models. Furthermore, these recursive trajectories provide a signal that raises membership inference TPR at $1\%$ FPR for nearly every attack we evaluate, roughly doubling it on the weakest baselines and still improving the strongest, which shows that a model's behavior under recursion carries membership evidence that a single query does not.
\end{abstract}

\section{Introduction}
The rapid development of generative AI has triggered a pressing demand for training data, frequently leading to the unauthorized usage of private, sensitive, or copyrighted content. Determining whether specific samples or
datasets were used to shape a model's parameters is therefore central to practical auditing, from protecting medical privacy~\citep{zhang2022healthmia} to identifying licensed content~\citep{dubinski2025cdi} and detecting benchmark
contamination~\citep{maini2024did,singh2024evaluation,zawalski2026codec}. Membership inference attacks (MIAs)~\citep{shokri2017mia} and dataset inference (DI)~\citep{maini2021di} formalize it, and a substantial body of work has developed increasingly refined scoring functions for these settings~\citep{zhang2024minkpp,wu2024yoqo,choquette2021labelonlymia}. What has not changed is the auditor's position: the model is trained, its weights and its training data out of reach, and all that remains is how it processes the information.

Both difficulties come from the fact that what a model has taken from its training data is distributed over billions of weights during training, which is difficult or expensive to repeat. Shadow-model attacks~\citep{carlini2022membership,ye2022enhanced} ask for that repetition, training auxiliary copies of the target to calibrate against, which is unaffordable at scale. A \singlequery attack asks instead for all of it in a single answer, and recent evaluations show its limitations, since many MIAs degrade sharply under distribution shift, at times barely beating chance~\citep{maini2024did,dubinski2025cdi}.

We think an audit at large scale needs two things. First, with no shadow model available, it has to be self-referential -- the only thing left to compare a model against is the model itself. Second, since one answer does not exhaust what the weights hold, it has to read more than one. Both follow from a recursive self-generation procedure, in which the model generates from the sample, then from that generation, and so on for several iterations; at every iteration it has nothing to condition on but its own previous output, so the trajectory that results is shaped entirely by what its weights already encode. One number becomes a trajectory, and the trajectory is what we study.

\begin{figure}[bt!]
  \centering
  \begin{subfigure}[b]{0.34\columnwidth}
      \centering
      \includegraphics[width=\linewidth]{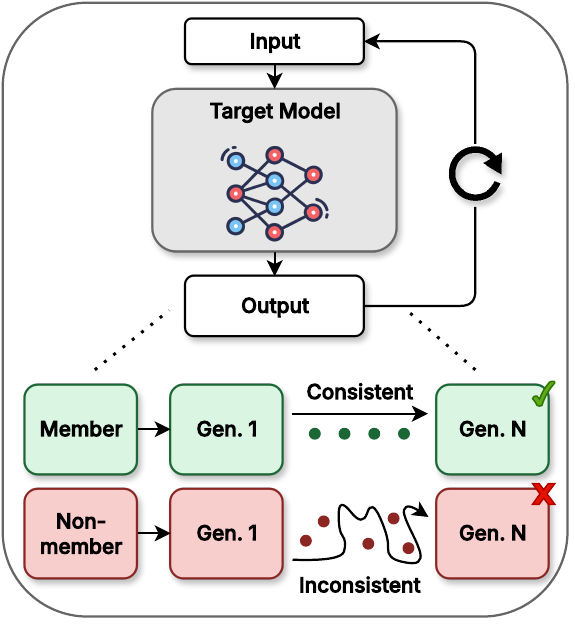}
       \vspace{0.2cm}
      \caption{Recursive self-generation.}
      \label{fig:procedure}
  \end{subfigure}
  \hfill
  \begin{subfigure}[b]{0.64\columnwidth}
      \centering
      \includegraphics[width=\linewidth]{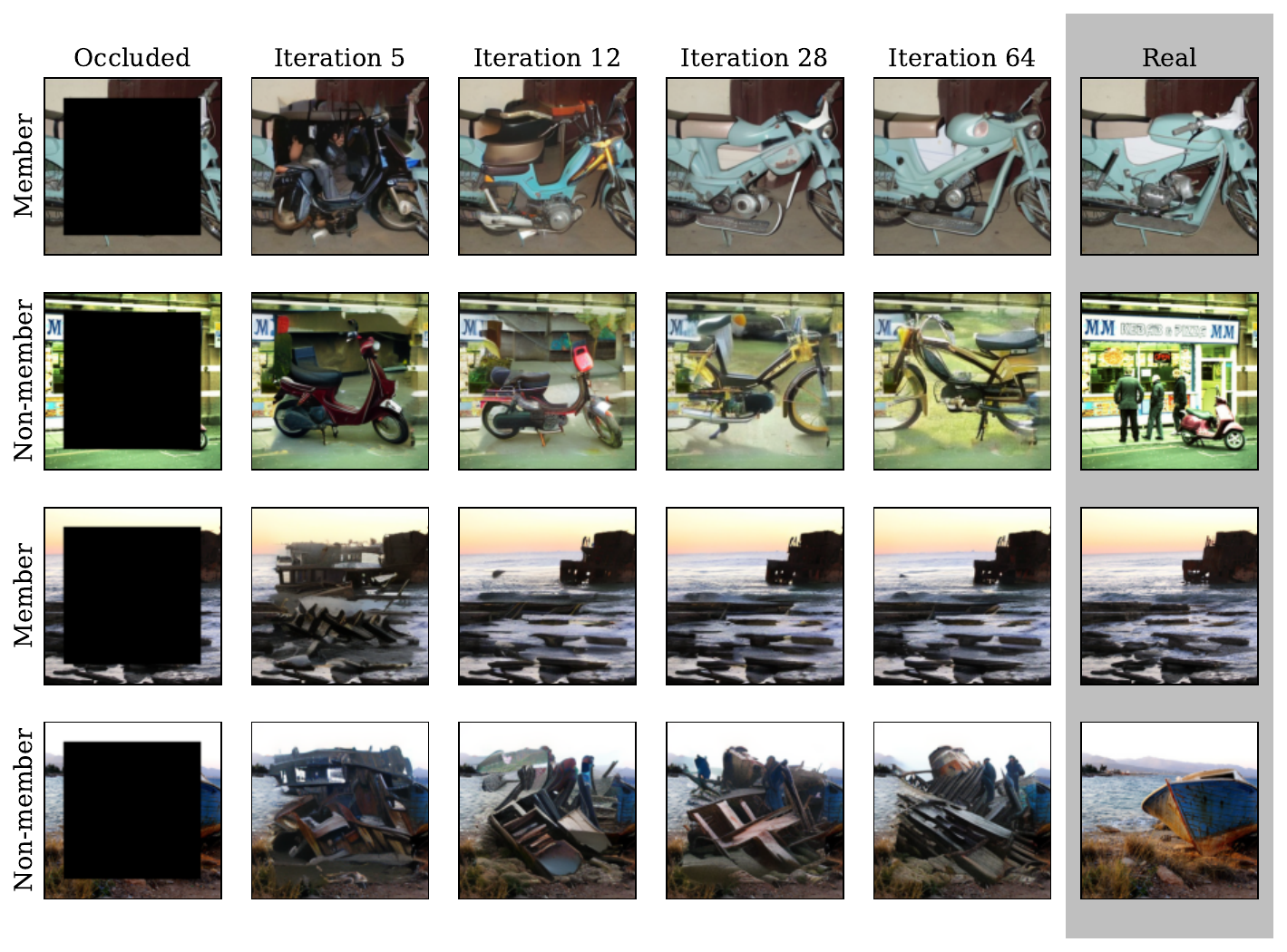}
      \caption{Anchored vs.\ drifting, at $60\%$ occlusion.}
      \label{fig:occluded}
  \end{subfigure}
  \caption{\textbf{Anchored or drifting under Recursive Self-Generation (RSG).} \textbf{(a)} Each iteration feeds the model's output back as its next input. With no external reference to correct toward, the trajectory is driven entirely by what the weights already encode. A training sample's trajectory stays near it; a held-out sample's trajectory is pulled toward the training manifold too, but to a neighborhood that is not its own. That difference is the signal we measure. \textbf{(b)} One example of the two behaviors, with $60\%$ of each image removed before the trajectory starts: both populations produce coherent images of the right class, but the training samples return to the particular image that was occluded while held-out samples settle on a different instance. The rightmost column shows the true original.}
  \label{fig:teaser}
  \vspace{-0.3cm}
\end{figure}

In those trajectories we observe a consistent asymmetry. One that starts from a sample the model was trained on (a member) stays close to it, while one that starts from a sample the model has not seen (a non-member) does not. Within a few iterations, the specifics of the real sample are gone, and it is drawn toward the training data itself, settling near samples the model was fitted on rather than near the one it started from -- visible in \Cref{fig:occluded} and measured in \Cref{fig:toy_geometry}. We call the two ends of this behavior \emph{anchored} and \emph{drifting}, and measure where a sample falls between them. The asymmetry holds across image autoregressive models, diffusion models, and large language models, and the gap is already measurable after a handful of iterations (\Cref{fig:fid-proof}).

Recursion of this kind has been studied before, as a pathology. Model Autophagy Disorder~\citep{alemohammad2023mad,shumailov2024mad2} describes how a \emph{sequence} of models retrained on synthetic data loses variance and collapses toward degenerate approximations of the data distribution. Ours is a single frozen model, with no retraining and recursion applied only at inference -- \emph{what is a failure mode during training becomes a measurement instrument after it}. We read the statistics of the trajectory to improve membership (MIA) and dataset inference (DI) attacks across model families and modalities. It costs $T$ generation calls per sample and nothing else -- no shadow models, and no access the base attack does not already have. 

In summary, our contributions are:

\begin{itemize}
   \item We identify an asymmetry in recursive self-generation: a member's trajectory stays anchored near the sample it started from, while a non-member's drifts toward the training data instead. Anchoring is graded rather than binary, and we show how to read it off a trajectory. The asymmetry holds across image autoregressive, diffusion, and language models.

\item We fuse trajectory features into existing attacks rather than replacing them, and the gain shows up at the low false-positive rates: TPR at $1\%$ FPR roughly doubles, improving in $118$ of $122$ attack -- model pairs, while AUC moves by about $5\%$.

\item We demonstrate that the gain is not a query-budget artifact. Spending equivalent compute on independent perturbations of the sample yields inferior results, confirming that the recursive trajectory captures membership evidence missed by isolated queries.
\end{itemize}

\section{Related Works}

\paragraph{Memorization.} Memorization in generative models -- the tendency to reproduce training examples rather than generate novel samples -- has been studied across model families, and the measurements are largely static. Early work formalized the distinction between memorization, mode collapse, and overfitting~\citep{burg2021memorization1}; subsequent studies characterized the generalization-to-memorization transition in diffusion models~\citep{gu2023memorization4}, localized it through attention patterns~\citep{ren2024memorization3}, and showed that standard evaluation
metrics fail to surface it~\citep{bai2021memorization2}. Mitigation has been proposed for both LLMs~\citep{hans2024memorization5} and text-to-image models~\citep{Chen2025memorization7}. What these measures share is that they evaluate a model in a single forward pass.
We instead read what a model retains of individual samples from its behavior over repeated self-generation, which surfaces structure that is not visible in any single pass.

\paragraph{Model Collapse.} Training a generative model on synthetic data has been studied as a training-time pathology. \citet{alemohammad2023mad} show that self-consuming loops progressively degrade quality and diversity when insufficient real data is injected, a phenomenon they term Model Autophagy Disorder, and \citet{shumailov2024mad2} show that recursive training on model-generated data causes tails of the original distribution to disappear. Both concern a \emph{sequence} of models retrained on synthetic data. We study a single frozen model, with no retraining, and apply recursion at inference time. Collapse is drift of a whole distribution across training generations; what we measure is whether an individual sample drifts or stays \emph{anchored} within the inference-time dynamics of one fixed model.

\paragraph{Membership and Dataset Inference.} Individual Membership Inference Attacks (MIA) can be confounded by distribution shifts~\citep{maini2024did}, prompting a shift toward Dataset Inference (DI), which aggregates evidence across many
samples~\citep{maini2021di,dubinski2025cdi,kowalczuk2025privacy}. Shadow model approaches~\citep{ye2022enhanced} are now computationally infeasible for large architectures, so modern attacks extract signals from limited black-box outputs~\citep{zhang2024minkpp,chang2025camia,tao2025informia}. Most relevant is \citet{li2024towards}, which averages independent perturbations of a target image. Because prior methods rely on isolated queries rather than a recursive feedback loop, our trajectory features capture a missing signal. Rather than competing with existing attacks, we seamlessly fuse our features into them.

\section{Method}
\label{sec:method}

\subsection{Recursive self-generation}
\label{sec:recursion}
One step of the procedure maps a sample to another sample of the same type -- a reconstruction for image models, a continuation for language models (\Cref{tab:operators}). We start from the real sample, and then the model's output becomes its own next input. Formally, given a sample $Z_0 = x$ and a trained generative model $f$, we construct
\[
Z_{t+1}=\mathcal{G}(f, Z_t),\qquad t=0,\dots,T-1,
\]
where $\mathcal{G}$ is a modality-specific generation operator: the procedure needs nothing of $f$ beyond its ability to generate a sample of the same type, which is what makes it applicable to any generative model and across modalities. The same operator and the same settings are used at every step, and the original sample $Z_0$ is never reintroduced after $t=0$. We refer to $(Z_0,\dots,Z_T)$ as the \emph{trajectory} of $x$ under $f$.

\subsection{Instantiating $\mathcal{G}$}
\label{sec:operators}

Table~\ref{tab:operators} summarizes the generation operator for each model family. In all cases, the operator is deterministic in its settings: sampling parameters, noise levels, and decoding configuration are fixed at $t=0$ and held
constant for the remainder of the trajectory.

\begin{table}[h]
\centering
\small
\caption{\textbf{The generation operator $\mathcal{G}$ across model families.} One step maps a sample to a new sample of the same type. Every setting, the strength $s$ included, is fixed at $t=0$ and held fixed for the remainder of the trajectory. The original sample is never reintroduced.}
\label{tab:operators}
\begin{tabular}{@{}l p{3.7cm} p{3.0cm} p{2.9cm} l@{}}
\toprule
\textbf{Family} & \textbf{One step of $\mathcal{G}$} & \textbf{Generation strength $s$} & \textbf{Held fixed} & \textbf{Access} \\
\midrule
IAR
& Regenerate part of the token representation and decode to an image
& Number of final scales (VAR) or fraction of tokens (RAR) regenerated
& Sampling temperature, class conditioning
& Black-box \\
\addlinespace
Diffusion
& Forward-noise to level $\sigma_s$, then reverse-denoise to $t=0$
& Re-noising level $\sigma_s$
& Scheduler, step count, guidance scale
& Black-box \\
\addlinespace
LLM
& Generate $m$ tokens from the current $k$-token prefix; advance the window by $m$
& Sampling temperature
& Window length $k$, stride $m$, decoding configuration
& Grey-box \\
\bottomrule
\end{tabular}
\end{table}

For image autoregressive models (IAR), one step is image-to-image generation under fixed decoding. For diffusion models, one step is partial forward noising to a fixed level followed by reverse denoising under a fixed scheduler. For language models, Transformer and state-space alike, one step is a sliding-window continuation: the model generates $m$ tokens conditioned on the current $k$-token prefix, and the window advances by dropping the oldest $m$ tokens and appending the newly generated ones. The context length stays fixed at $k$ while generated text progressively displaces the original prefix, so that after $\lceil k/m \rceil$ steps no original tokens remain.

We define \emph{generation strength} as the magnitude of the model's intervention at each recursive step, quantifying the balance between the information preserved from the previous state and the new information introduced from the model's learned distribution. Each operator carries a single control $s$ governing how much of $Z_t$ is discarded and resampled: the number of final scales (VAR) or fraction of tokens (RAR) re-generated for IARs; the re-noising level $\sigma_s$ for diffusion models; and the temperature parameter, for language models. Larger $s$ means less of the previous state is carried over.

\subsection{Measuring anchoring and drift}
\label{sec:measurements}
Let $\delta_t := \delta(Z_0, Z_t)$ measure how far the trajectory has moved from the sample it started from. Under recursion $\delta_t$ rises and then flattens, so a trajectory is summarized by two numbers: the \emph{contraction rate} $\lambda\in[0,1]$ at which it flattens, smaller meaning stronger anchoring, and the \emph{drift floor} $\delta_\infty$ it flattens at. A sample is \emph{anchored} when $\delta_\infty$ is small -- the model's own dynamics hold it near where it started -- and \emph{drifting} when it is large, so the two words name the ends of a scale rather than two classes, and the samples an audit finds hardest sit between them. Corollary~\ref{cor:floor} shows these are not free parameters of the measurement: at a memorized point $\lambda = 1-s^2 m$, with $s$ the generation strength we choose and $m$ the local sharpness of the model's density there. The trajectory features we feed the scorer estimate $\lambda$, $\delta_\infty$, or functions of both. The sequence $\{\delta_t\}_{t=1}^{T}$ is what the term \emph{anchoring} refers to throughout the paper, and it is measured per sample. 

We instantiate $\delta$ per modality: MSE, LPIPS~\citep{zhang2018lpips}, and SSIM~\citep{wang2004ssim} for images, and 12 features for language models, e.g. Mean Reciprocal Rank, Edit distance, or Jensen-Shannon divergence. We provide detailed feature descriptions in Appendix~\ref{app:llm-features} and~\ref{app:vision-features}.

\subsection{From anchoring to membership inference}
\label{sec:mia-application}

 We score each sample twice under an identical procedure, once from the \singlequery evidence alone, and once from the same evidence together with the trajectory it generates, and compare the two. The scorer, the calibration protocol, the splits, and the base attacks are the same in both cases.

Let $z_{\mathrm{base}}=\phi_{\mathrm{base}}(Z_0)\in\mathbb{R}^{d}$ denote the score produced by any \singlequery attack. We form the fused representation $\tilde z=[z_{\mathrm{base}}\|z_{\mathrm{traj}}]\in\mathbb{R}^{d+k}$ and fit an $L_1$-regularized logistic regression $h$ as a plug-in estimator of $P(M=1\mid\tilde z)$, giving the final score $s(Z_0)=h(\tilde z)$. For dataset inference, the procedure is identical, with the membership label replaced by dataset
origin $D\in\{0,1\}$ and per-sample scores aggregated by a fixed mean rule. Feature orientation, normalization, and the regression itself are fit on the training split alone and applied unchanged at test time. The auditor requires only the access needed by the base attack plus $T$ generation calls per sample. Because recursive self-generation relies entirely on standard model outputs, it functions as an any-box method. The effective access level is therefore inherited directly from the fused attack. 

\section{Experiments}
\label{sec:experiments}

\subsection{Experimental Setup}
To ensure a scientifically sound evaluation across our MIA tasks, we restrict our setup to models trained on public datasets with well-defined training and test splits. We evaluate our method on three model families spanning two modalities, vision and language, to demonstrate its broad applicability. For image generation, we analyze SOTA autoregressive models (VAR-d\{20, 24, 30\}~\citep{tian2024var}, RAR-\{L, XL, XXL\}~\citep{yu2024rar}) and diffusion models DiT-MoE-\{XL, G\}~\citep{fei2024dit}, UViT-T2I-Deep~\citep{bao2022uvit}), trained primarily on the ImageNet~\citep{deng2009imagenet} or COCO~\citep{lin2014microsoft} datasets. We extend this evaluation to the language domain utilizing prominent LLMs (Pythia-\{6.9, 12\}B~\citep{biderman2023pythia}, GPT-NeoX-20B~\citep{black2022gptneox}, GPT-Neo-2.7B~\citep{black2021gptneo}, GPT-J-6B~\citep{wang2021gptj}, and the state-space model Mamba-2.8B~\citep{gu2023mamba}). We leverage the established Mimir~\citep{duan2024mimir} dataset on~\textit{ngram\_7\_0.2} split to evaluate language models. Comprehensive details regarding all specific models and datasets used in experiments are provided in the Appendix~\ref{app:model_details} and~\ref{app:dataset_details}. 

\subsection{Protocol and Metrics}
\label{sec:metrics}
\paragraph{Metrics.}
In the main part of our experiments we report TPR@$1\%$FPR as the most important metric for MIAs, since an audit is only useful at an operating point where a positive call can be acted on, and threshold-free summaries average over regimes an auditor would never use; AUC is reported beside it in the appendix (\Cref{tab:diffusion_mia_auc}, \Cref{tab:iar_mia_auc}). To measure the similarity between feature representations and their fidelity, we utilize the Fréchet Inception Distance (FID)~\citep{heusel2017fid} for vision, with $\Delta\text{FID}=\text{FID}_{\text{nonmem}}-\text{FID}_{\text{mem}}$ for the population gap. For language models we use Kullback–Leibler Divergence (KLD) between the normalized average token probability distribution at the current iteration and that of the first evaluation iteration. For additional insights into trajectory dynamics we also report generative precision and recall~\citep{kynkaanniemi2019improved} metrics: precision is the share of generated samples that fall inside the manifold of real ones, recall the share of real samples covered by the generated manifold. We also use
Recovery@1 (\Cref{sec:recovery}) -- the rate at which the final iterate ranks the true original first among same-class distractors.

\paragraph{MIA pipeline.}
For each labeled member/non-member sample, we generate $Z_0,\dots,Z_T$, compute $\phi_t$, $S_T$, and modality-specific trajectory features, form $\tilde z$, and fit $h_{\mathrm{mia}}$. We evaluate univariate trajectory statistics (as described in~\Cref{sec:measurements}) by direct thresholding and multivariate features by logistic-regression fusion on strictly stratified train-test splits. We report AUC and TPR@$1\%$FPR.
All descendants of the same $Z_0$ (all $Z_t$, all derived features) remain in the same partition. Thresholds, feature normalization, and LR calibration are fit on train only.

\begin{figure}[t!]
\centering
\begin{minipage}[t]{0.32\textwidth}
    \centering
    \subtitleboxblack{\textbf{IAR}\hspace{0.5em}\emph{(black-box)}}
    \vspace{0.3em}

    \begin{subfigure}[b]{0.9\linewidth}
        \includegraphics[width=\linewidth]{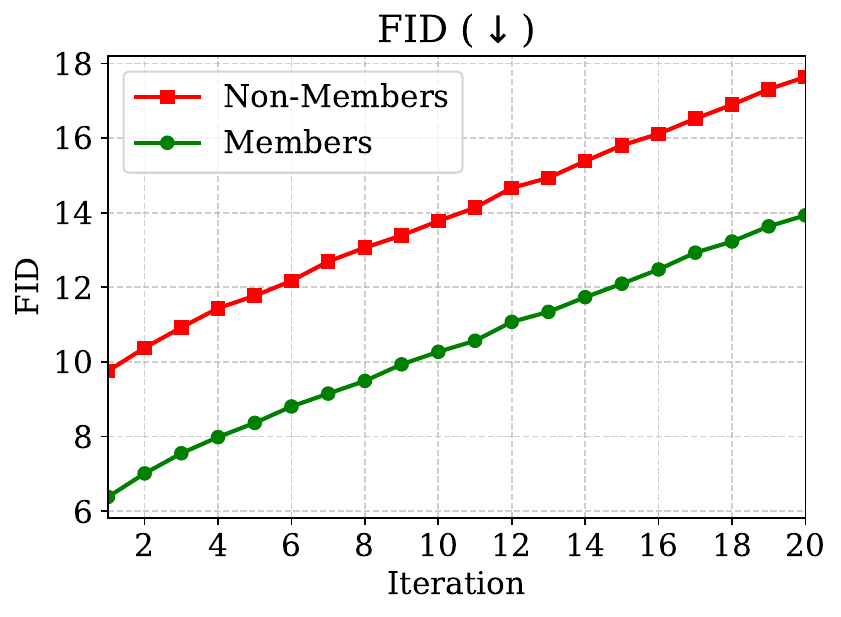}
        \caption{VAR-d30}
    \end{subfigure}
\end{minipage}
\hfill
\begin{minipage}[t]{0.32\textwidth}
    \centering
    \subtitleboxblack{\textbf{Diffusion}\hspace{0.5em}\emph{(black-box)}}
    \vspace{0.3em}

    \begin{subfigure}[b]{0.9\linewidth}
        \includegraphics[width=\linewidth]{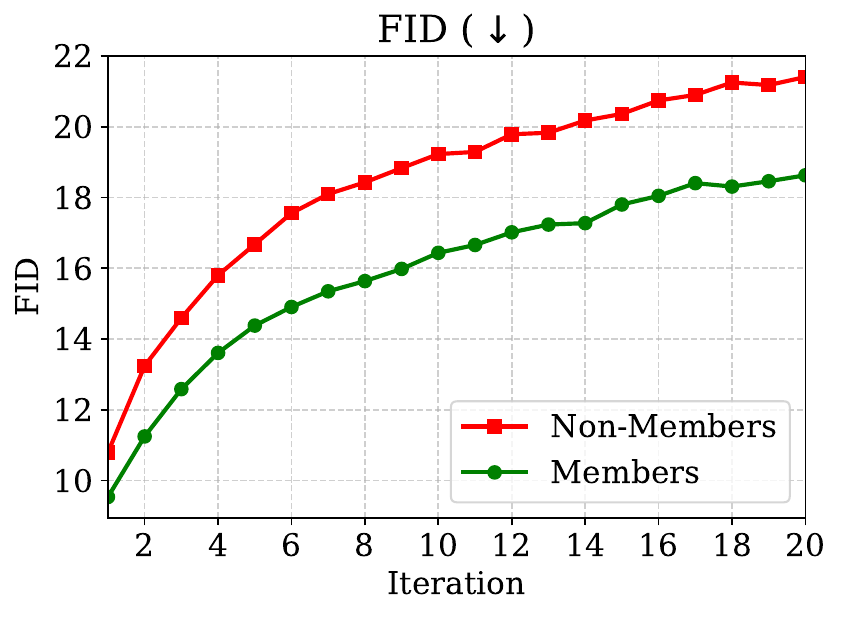}
        \caption{DiT-MoE-G}
    \end{subfigure}
\end{minipage}
\hfill
\begin{minipage}[t]{0.32\textwidth}
    \centering
    \subtitlebox{\textbf{Language}\hspace{0.5em}\emph{(grey-box)}}
    \vspace{0.3em}

    \begin{subfigure}[b]{0.9\linewidth}
        \includegraphics[width=\linewidth]{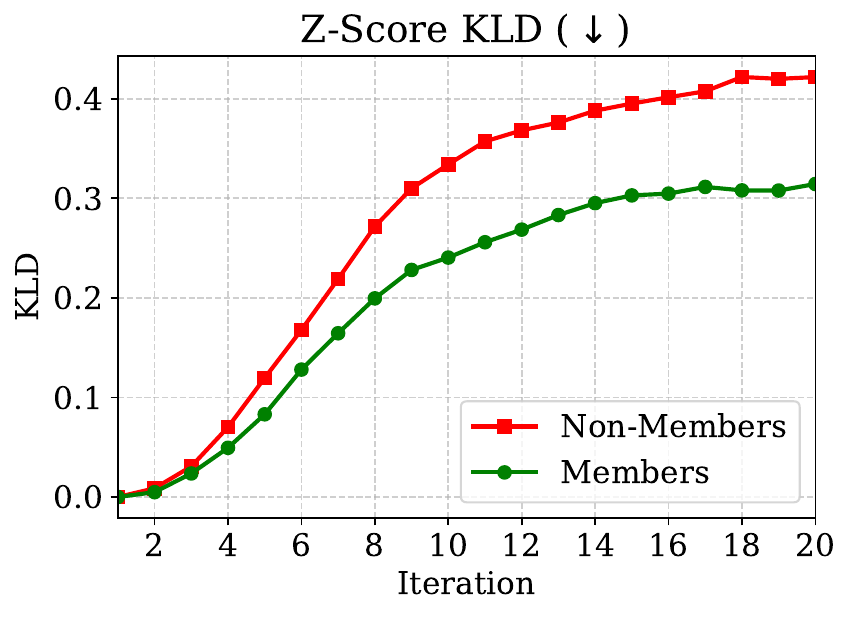}
        \caption{Pythia-6.9B}
    \end{subfigure}
\end{minipage}

\caption{\textbf{Population-level divergence under recursive self-generation.} Each curve tracks a whole set of samples: at every iteration we compute FID (vision) or KLD (language) between the set of iterates and the corresponding set of originals. These are distributional distances between populations. The member population diverges more slowly than the non-member population throughout, establishing the asymmetry that the per-sample measurements exploit.}
\label{fig:fid-proof}
\vspace{-0.5cm}
\end{figure}

Applying $\mathcal{G}$ repeatedly to a sample produces a trajectory that either stays near where it started or does not, and which of the two occurs is governed by whether the sample was in the training set. \Cref{fig:fid-proof} shows this for one model from each family. Divergence from $Z_0$ grows in both cases (recursion is lossy), but it grows consistently slower for training samples, and the gap widens within the first few iterations and persists to the end of the trajectory.

The mechanism is visible in what recursion removes. Recursive self-generation has no external reference to correct toward, because at every step the model can only condition on its own previous output, so whatever survives must be sustained by its weights alone. Samples memorized by the model have somewhere to be sustained, and the trajectory returns to them, while samples it has not seen have no such point, and the trajectory settles wherever the model's priors place it.

\subsection{Anchoring and drift in Controlled DDPMs}
\label{sec:toy}

We further present a small, fully controlled experiment to test whether \emph{repeatedly}
passing an image through a generative model's own inference pipeline
-- rather than looking at a single reconstruction error -- exposes
membership in the training set. To that end, we train a family of small CIFAR-10 DDPMs from scratch at six deliberately different levels of memorization and probe each of them with recursive self-generation under a frozen SDEdit operator (i.e., a fixed-weight model whose sample is repeatedly re-noised to an intermediate diffusion step $t_{edit}$ and denoised back, so no training takes place between iterations). Our setup lets us disentangle two questions that are usually
confounded: how memorization shapes the geometry of the trajectory, and how much of the resulting membership signal can be recovered by a classical attacker versus one that reads the whole trajectory. Here the generation strength $s$ of \Cref{tab:operators} is instantiated as $\mathrm{denoise\_steps}$, the DDIM reverse-diffusion budget spent on each iteration, which sets how much of the input the model rewrites per pass.

\begin{figure*}[bt!]
    \centering
    \begin{subfigure}[t]{0.32\textwidth}
        \centering
        \includegraphics[width=\linewidth]{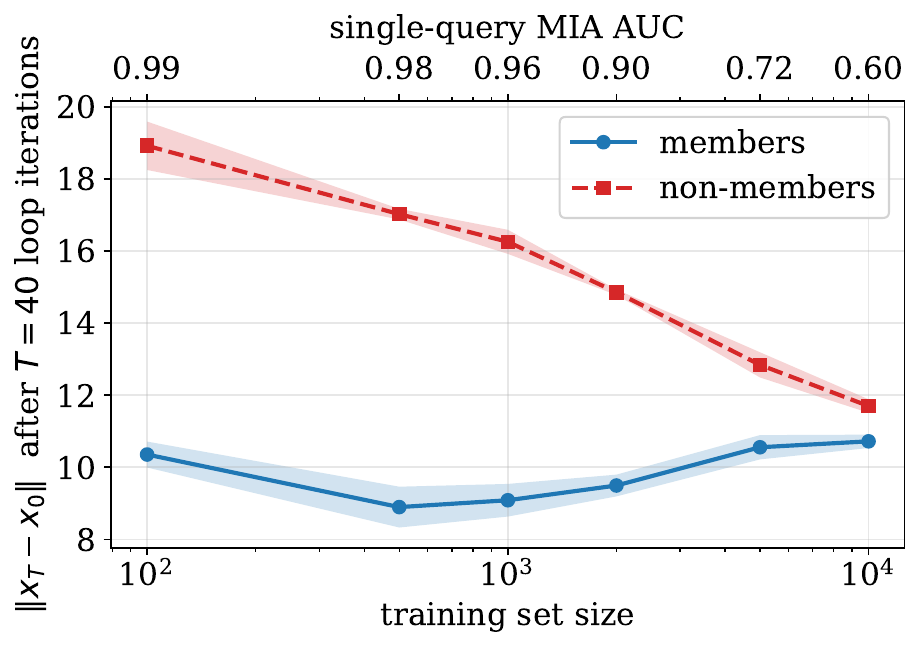}
               \caption{Endpoint displacement $\|x_T-x_0\|$ (pixels).
        Members stay \emph{anchored}; non-members \emph{drift},
        traveling up to $2\times$ further.}
        \label{fig:toy_geometry_displacement}
    \end{subfigure}
    \hfill
    \begin{subfigure}[t]{0.32\textwidth}
        \centering
        \includegraphics[width=\linewidth]{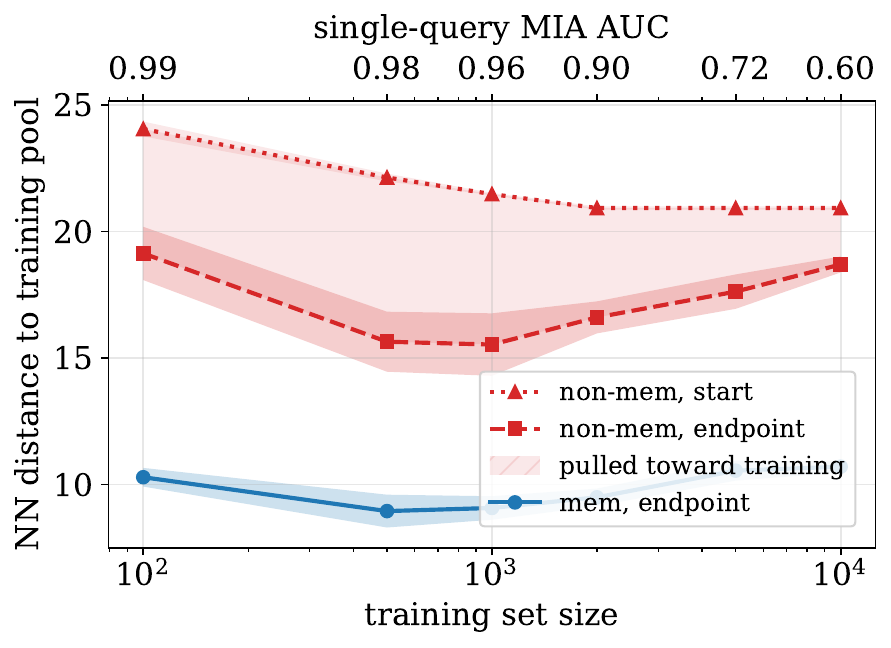}
                \caption{Nearest training-image distance, start vs.\ end.
        Non-members are \emph{pulled toward training}; members
        stay on the training manifold.}
        \label{fig:toy_geometry_nn}
    \end{subfigure}
    \hfill
    \begin{subfigure}[t]{0.32\textwidth}
        \centering
        \includegraphics[width=\linewidth]{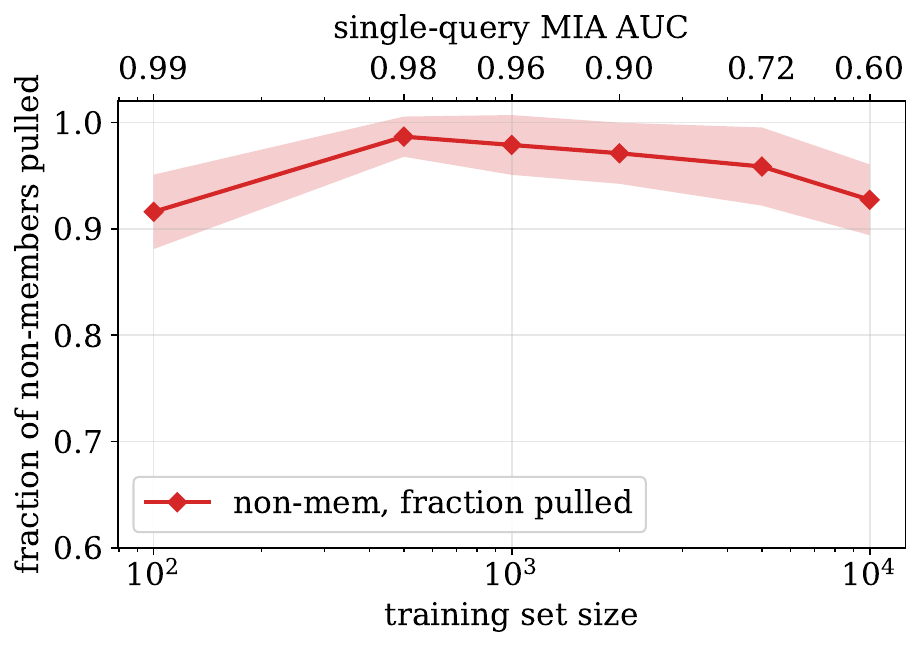}
                \caption{Fraction of non-members ending closer to the training
        pool than $x_0$ did. Remains $\ge 0.9$ across all
        training-set sizes.}
        \label{fig:toy_geometry_frac}
    \end{subfigure}
        \caption{\textbf{Geometry of frozen DDPM trajectories across training-set size.} Small-CNN DDPMs trained from scratch for 400 epochs on $\{100, 500, 1000, 2000, 5000, 10000\}$ disjoint CIFAR-10 images. The architecture is identical across runs, so the degree of memorization is set by training-set size alone; the top axis of each panel reports the resulting \singlequery MIA AUC (pixel-MSE at
    $\mathrm{denoise\_steps}=10$), which declines as the training set grows. Every trajectory is iterated at $\mathrm{denoise\_steps}=2$ for $T=40$ iterations; shaded bands are $\pm 1$ SD across 5 seeds.}
    \label{fig:toy_geometry}
    \vspace{-0.5cm}
\end{figure*}
\subsection{Training data stays anchored under recursive self-generation}
\label{sec:drift}

\textbf{Members are fixed points, non-members are transported.}Recursion induces a direction asymmetry between the two
populations (Fig.~\ref{fig:toy_geometry}). A member image barely
moves -- the endpoint sits on average $\|x_T-x_0\|\approx 10$ pixels from the input, essentially flat across the memorization range -- while
a non-member is transported $1.5$--$2\times$ further and, crucially,
not in a random direction: for every training-set size, over $90\%$ of 
non-member trajectories terminate closer to the training pool than
they began (Fig.~\ref{fig:toy_geometry_frac}), and their
nearest-neighbour pixel distance to that pool shrinks along the trajectory, while members are on the pool from the outset and stay there (Fig.~\ref{fig:toy_geometry_nn}).

The same trajectories also carry the membership signal. A logistic regression over the sequence of per-step origin distances lifts TPR@$5\%$FPR from $0.43$ at $K=1$, which is the classical \singlequery attack, to $0.99$ at $K=15$, while AUC is already above $0.9$ at $K=1$ and merely saturates (Appendix~\ref{app:toy-classifier}).

\label{sec:strength}
\begin{figure}[bt!]
  \centering
  \includegraphics[width=\linewidth]{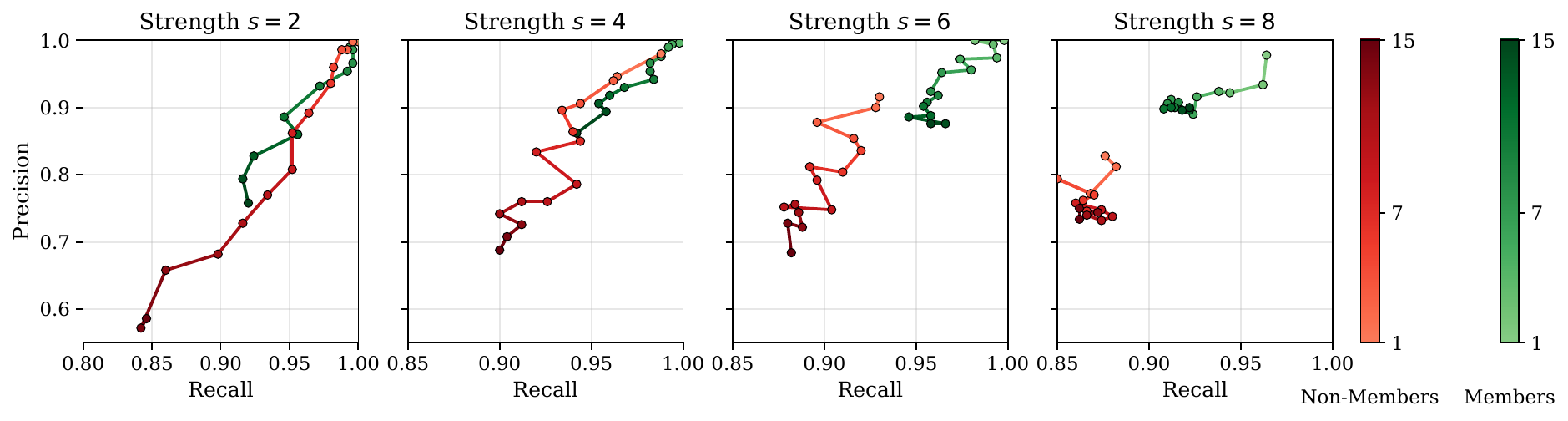}
  \caption{\textbf{Generative Precision-Recall curves for VAR-d30} across generation strengths $s \in \{2, 4, 6, 8\}$. Members (green) and Non-Members (red) are traced over 15 iterations, with color intensity indicating iteration progress. Larger $s$ corresponds to more aggressive regeneration.}
  \label{fig:prec-rec}
\end{figure}

\begin{figure}[t]
\centering
\begin{minipage}[t]{0.56\textwidth}
  \vspace{0pt}
  \centering
  \includegraphics[width=\linewidth]{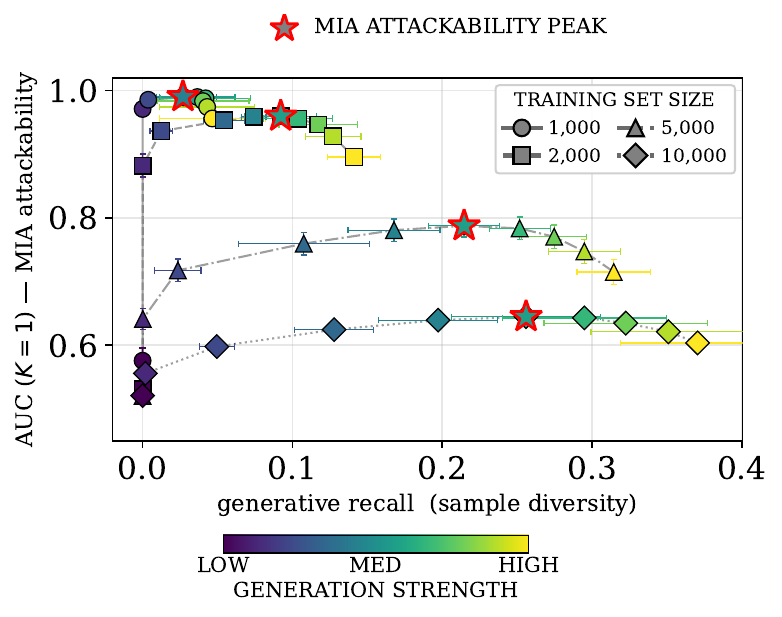}
  \caption{\textbf{MIA attackability across generation strength and memorization.} Markers denote (training-set size, $\mathrm{denoise\_steps}$) pairs, averaged over five seeds. Each model traces an arc with a single interior maximum.}
  \label{fig:toy_tradeoff}
\end{minipage}
\hfill
\begin{minipage}[t]{0.42\textwidth}
  \vspace{20pt}
  \centering
  \small
  \captionof{table}{\textbf{Generation strength on RAR-XXL} (TPR@$1\%$FPR; best per row in bold).
Four of the five attacks peak at the interior $s{=}0.5$ on RAR-XXL rather than at the
most aggressive setting.}
  \label{tab:korba_rar_xxl}
  \setlength{\tabcolsep}{2.2pt}
  \renewcommand{\arraystretch}{1.05}
  \begin{tabular}{l|c|ccc}
\toprule
 & & \multicolumn{3}{c}{+ours} \\
\cmidrule(lr){3-5}
Attack & Baseline & $s{=}0.2$ & $s{=}0.5$ & $s{=}0.8$ \\
\midrule
Loss & 5.9 & 9.0 & 9.0 & \textbf{9.1} \\
Hinge & 5.7 & 7.8 & \textbf{9.1} & 8.9 \\
Min-K\%++ & 13.0 & 16.0 & \textbf{17.8} & 17.4 \\
SURP & 4.3 & 6.7 & \textbf{8.3} & 7.0 \\
CAMIA & 3.6 & 5.2 & \textbf{7.0} & 6.2 \\
\bottomrule
\end{tabular}
\end{minipage}
\end{figure}

\subsection{Recursion returns to specific samples}
\label{sec:recovery}

Under class-conditional generation, a trajectory might stay close to $Z_0$ simply by producing a generic sample of the correct class. If most of a sample is destroyed before the trajectory begins, does the model reconstruct \emph{that sample} or \emph{its class}? We test this by heavily occluding the image and forcing the model to inpaint the missing content.\begin{wrapfigure}[11]{r}{0.42\columnwidth}
\vspace{-60pt}
\centering
  \includegraphics[width=\linewidth]{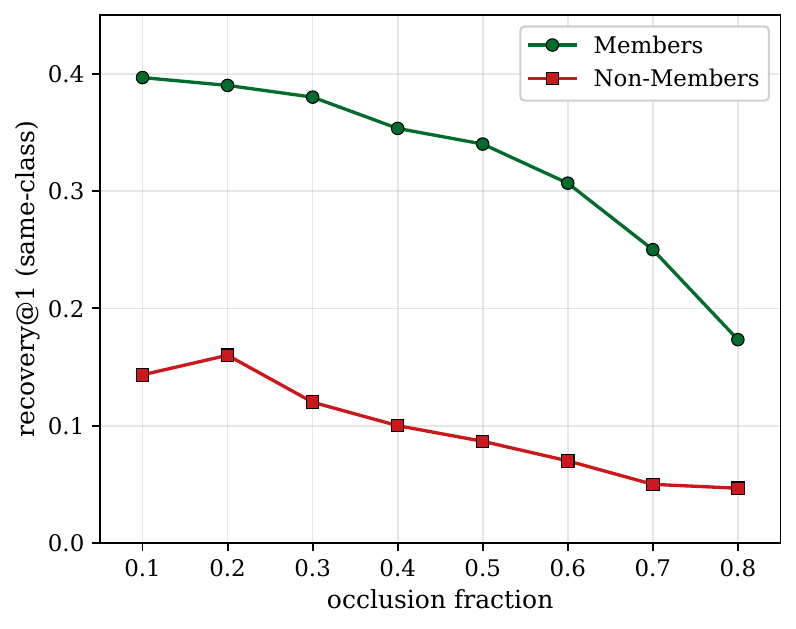}
  \caption{\textbf{Recovery of the occluded original}, on DiT-MoE-XL. The final iterate is ranked against the true original and ten same-class distractors.}
  \label{fig:recovery}
\end{wrapfigure}

Qualitatively (\Cref{fig:occluded}), while both populations generate coherent class images under 60\% occlusion, member trajectories reconstruct the specific missing objects whereas non-member trajectories land on a different instance of the same class. We quantify this by ranking the final iterate against the true original and 10 same-class distractors, eliminating any category-level advantage. On DiT-MoE-XL (\Cref{fig:recovery}), training samples are retrieved at 0.39 and 0.31 under 20\% and 60\% occlusion, compared to 0.16 and 0.07 for non-members.

\subsection{Generation strength controls the drift rate}
Figure~\ref{fig:prec-rec} shows generative precision--recall curves~\citep{kynkaanniemi2019improved} for VAR-d30 at $s \in \{2, 4, 6, 8\}$, where $s$ is the number of final scales regenerated. Members separate from non-members at every $s$. \Cref{tab:korba_rar_xxl} sweeps $s$ for the attack itself on RAR-XXL, and \Cref{tab:korba_var30,tab:korba_var24,tab:korba_pythia_6_9b,tab:korba_pythia_12b} present results for the remaining models.

\textbf{More generation strength is not always more membership signal.} While \Cref{fig:prec-rec} might suggest that stronger recursion always improves performance by driving populations further apart, this is not the case. Evaluated on actual attacks, performance consistently shows a single interior maximum rather than peaking at either extreme (\Cref{fig:toy_tradeoff} and \Cref{tab:korba_rar_xxl}). This tradeoff arises because separating broad populations differs from ranking individual samples. Weak regeneration yields a trajectory too subtle to measure, whereas strong regeneration overwrites the input with a fresh prior sample. Under strong regeneration, populations still separate but individual trajectories lose their sample-specific information. Generation strength is thus an operating point to tune rather than a parameter to maximize. Theory mirrors this shape in Corollary~\ref{cor:floor} and Remark~\ref{rem:sigma}, showing that raising $s$ accelerates contraction but lifts the floor $\varepsilon_s/(1-\lambda)$ and narrows the valid contraction region.

\subsection{Trajectories recover membership signal where single-forward attacks fail}
\label{sec:mia}
\begin{table}[tb]
\centering
\scriptsize
\caption{\textbf{MIA on IARs} (TPR@$1\%$FPR).}
\label{tab:iar_mia_tpr}
\resizebox{\linewidth}{!}{%
\begin{tabular}{lcccccccc}
\toprule
 & \multicolumn{2}{c}{VAR-d30} & \multicolumn{2}{c}{VAR-d24} & \multicolumn{2}{c}{RAR-XXL} & \multicolumn{2}{c}{RAR-XL} \\
\cmidrule(lr){2-3} \cmidrule(lr){4-5} \cmidrule(lr){6-7} \cmidrule(lr){8-9}
Attack & Baseline & +ours & Baseline & +ours & Baseline & +ours & Baseline & +ours \\
\midrule
Loss & 3.5 & 6.6~\dlt{3.1} & 1.8 & 2.3~\dlt{0.5} & 5.9 & 9.0~\dlt{3.1} & 3.2 & 5.0~\dlt{1.8} \\
Hinge & 2.8 & 5.0~\dlt{2.2} & 1.8 & 2.0~\dlt{0.2} & 5.7 & 9.1~\dlt{3.4} & 3.4 & 5.0~\dlt{1.6} \\
Min-K\%++ & 9.5 & 15.5~\dlt{6.0} & 2.7 & 3.4~\dlt{0.7} & 13.0 & 17.8~\dlt{4.8} & 6.0 & 8.6~\dlt{2.6} \\
SURP & 3.6 & 6.4~\dlt{2.8} & 1.8 & 2.3~\dlt{0.5} & 4.3 & 8.3~\dlt{4.0} & 2.4 & 3.7~\dlt{1.3} \\
CAMIA & 15.2 & 14.8~\dlt{-0.4} & 2.3 & 3.9~\dlt{1.6} & 3.6 & 7.0~\dlt{3.4} & 2.2 & 3.8~\dlt{1.6} \\
\midrule
SOTA & 92.0 & 92.5~\dlt{0.5} & 22.8 & 25.2~\dlt{2.4} & 44.4 & 47.3~\dlt{2.9} & 25.5 & 27.1~\dlt{1.6} \\
\bottomrule
\end{tabular}
}
\end{table}
\begin{wraptable}[12]{r}{0.44\columnwidth}
\centering
\scriptsize
\setlength{\tabcolsep}{4pt}
\caption{\textbf{MIA on diffusion models} (TPR@$1\%$FPR).}
\label{tab:diffusion_mia_tpr}
\begin{tabular}{lcccc}
\toprule
 & \multicolumn{2}{c}{DiT-MoE-XL} & \multicolumn{2}{c}{UViT-T2I-Deep} \\
\cmidrule(lr){2-3} \cmidrule(lr){4-5}
Attack & Base & +ours & Base & +ours \\
\midrule
SecMI & 2.0 & 17.9~\dlt{15.9} & 2.3 & 3.3~\dlt{1.0} \\
Den. Loss & 1.4 & 26.7~\dlt{25.3} & 2.0 & 3.3~\dlt{1.3} \\
PIA & 2.9 & 24.5~\dlt{21.6} & 4.3 & 7.0~\dlt{2.7} \\
PIAN & 11.9 & 17.3~\dlt{5.4} & 0.9 & 1.7~\dlt{0.8} \\
CDI & 74.8 & 77.0~\dlt{2.2} & 5.9 & 8.7~\dlt{2.8} \\
\bottomrule
\end{tabular}
\end{wraptable}
Anchoring is therefore real, instance-specific, and measurable at a tunable strength (\Cref{sec:drift,sec:recovery,sec:strength}). To test whether this signal is practically useful, we take the least invasive approach available. We leave existing attacks unmodified and simply concatenate their scores with our trajectory features before applying a calibrated scorer. Any improvement is therefore attributable to information in the trajectory that the base attack does not already extract.

We report TPR at $1\%$ FPR, and include a matched-compute control in which the same query budget is spent on independent perturbations of $Z_0$ rather than on recursion. We benchmark language models against Loss~\citep{yeom2018loss}, Zlib~\citep{carlini2021zlib}, Hinge~\citep{bertran2023hinge}, Min-K\%~\citep{shi2023mink}, Min-K\%++~\citep{zhang2024minkpp}, ReCaLL~\citep{xie2024recall}, SURP~\citep{zhang2024adaptive}, and CAMIA~\citep{chang2025camia}, alongside their adapted versions for image autoregressive models~\citep{kowalczuk2025privacy}, whose strongest attack per model we report as SOTA. For diffusion models, our benchmark encompasses SecMI~\citep{duan2023secmi}, Denoising Loss~\citep{carlini2023extracting}, PIA~\citep{kong2023pia}, PIAN, and CDI~\citep{dubinski2025cdi}.

\begin{table}[ht!]
\centering
\caption{\textbf{MIA on Pythia-6.9B, per MIMIR domain} (TPR@$1\%$FPR).}
\label{tab:llm_mia_domain_pythia_6_9b}

\resizebox{\textwidth}{!}{%
\begin{tabular}{lcccccccc}
\toprule
 & \multicolumn{2}{c}{arXiv} & \multicolumn{2}{c}{DM Math} & \multicolumn{2}{c}{GitHub} & \multicolumn{2}{c}{HackerNews} \\
\cmidrule(lr){2-3} \cmidrule(lr){4-5} \cmidrule(lr){6-7} \cmidrule(lr){8-9}
Attack & Baseline & +ours & Baseline & +ours & Baseline & +ours & Baseline & +ours \\
\midrule
Loss & 3.1 & 11.0~\dlt{7.9} & 30.8 & 75.9~\dlt{45.1} & 31.1 & 33.4~\dlt{2.3} & 2.4 & 5.2~\dlt{2.8} \\
Zlib & 2.7 & 10.4~\dlt{7.7} & 11.6 & 74.2~\dlt{62.6} & 43.5 & 38.8~\dlt{-4.7} & 2.4 & 4.9~\dlt{2.5} \\
Hinge & 0.6 & 9.1~\dlt{8.5} & 20.3 & 75.2~\dlt{54.9} & 23.4 & 24.8~\dlt{1.4} & 2.2 & 5.0~\dlt{2.8} \\
Recall & 8.5 & 15.5~\dlt{7.0} & 21.6 & 76.2~\dlt{54.6} & 38.8 & 33.3~\dlt{-5.5} & 3.2 & 6.4~\dlt{3.2} \\
Min-K\% & 4.6 & 15.5~\dlt{10.9} & 57.5 & 74.8~\dlt{17.3} & 30.5 & 32.4~\dlt{1.9} & 2.3 & 5.6~\dlt{3.3} \\
Min-K\%++ & 2.9 & 16.3~\dlt{13.4} & 29.2 & 75.4~\dlt{46.2} & 22.0 & 29.9~\dlt{7.9} & 2.7 & 6.6~\dlt{3.9} \\
SURP & 4.3 & 14.0~\dlt{9.7} & 57.9 & 72.7~\dlt{14.8} & 31.2 & 32.0~\dlt{0.8} & 2.4 & 5.8~\dlt{3.4} \\
CAMIA & 11.3 & 14.0~\dlt{2.7} & 19.0 & 74.8~\dlt{55.8} & 26.0 & 25.8~\dlt{-0.2} & 4.7 & 7.2~\dlt{2.5} \\
\end{tabular}%
}

\resizebox{\textwidth}{!}{%
\begin{tabular}{lcccccccc}
\toprule
 & \multicolumn{2}{c}{Pile CC} & \multicolumn{2}{c}{PubMed} & \multicolumn{2}{c}{Wikipedia} & \multicolumn{2}{c}{Average} \\
\cmidrule(lr){2-3} \cmidrule(lr){4-5} \cmidrule(lr){6-7} \cmidrule(lr){8-9}
Attack & Baseline & +ours & Baseline & +ours & Baseline & +ours & Baseline & +ours \\
\midrule
Loss & 6.3 & 11.3~\dlt{5.0} & 8.2 & 16.9~\dlt{8.7} & 10.6 & 13.9~\dlt{3.3} & 13.2 & 23.9~\dlt{10.7} \\
Zlib & 6.1 & 7.3~\dlt{1.2} & 6.1 & 16.6~\dlt{10.5} & 6.6 & 13.4~\dlt{6.8} & 11.3 & 23.7~\dlt{12.4} \\
Hinge & 4.0 & 9.9~\dlt{5.9} & 3.4 & 10.2~\dlt{6.8} & 8.1 & 11.4~\dlt{3.3} & 8.9 & 20.8~\dlt{11.9} \\
Recall & 4.6 & 7.2~\dlt{2.6} & 8.5 & 19.6~\dlt{11.1} & 8.8 & 14.6~\dlt{5.8} & 13.4 & 24.7~\dlt{11.3} \\
Min-K\% & 6.4 & 10.4~\dlt{4.0} & 14.9 & 21.7~\dlt{6.8} & 10.7 & 12.5~\dlt{1.8} & 18.1 & 24.7~\dlt{6.6} \\
Min-K\%++ & 5.0 & 7.8~\dlt{2.8} & 8.9 & 16.3~\dlt{7.4} & 6.6 & 8.6~\dlt{2.0} & 11.0 & 23.0~\dlt{12.0} \\
SURP & 7.2 & 11.5~\dlt{4.3} & 14.8 & 21.2~\dlt{6.4} & 10.9 & 13.3~\dlt{2.4} & 18.4 & 24.3~\dlt{5.9} \\
CAMIA & 4.5 & 15.1~\dlt{10.6} & 5.9 & 14.3~\dlt{8.4} & 8.9 & 11.7~\dlt{2.8} & 11.5 & 23.3~\dlt{11.8} \\
\bottomrule
\end{tabular}%
}
\end{table}
\begin{table}[t]
\centering
\caption{\textbf{Recursion is not a query-budget effect} (TPR@$1\%$FPR, best in bold). The matched-compute \emph{+aug} control spends the same budget on $T$ independent perturbations of
$Z_0$ instead of recursion. \textbf{(a)} It degrades accuracy against the strongest attacks per family. \textbf{(b)} Even on attacks it benefits, the trajectory still outperforms it. Language results average the seven MIMIR domains.}
\label{tab:aug-pair}
\begin{subtable}[t]{0.48\textwidth}
\centering
\scriptsize
\setlength{\tabcolsep}{3.5pt}
\caption{Strongest attack available.}
\label{tab:aug-ablation}
\begin{tabular}{lccc}
\toprule
Family & IAR & Diffusion & Language \\
Model & RAR-XXL & DiT-MoE-XL & Pythia-6.9B \\
\midrule
Attack & SOTA & CDI & SURP \\
Baseline & 44.4 & 74.8 & 18.4 \\
\;\;+aug & 43.6~\dlt{-0.8} & 72.3~\dlt{-2.5} & 10.2~\dlt{-8.2} \\
\;\;+ours & \textbf{47.3}~\dlt{2.9} & \textbf{77.0}~\dlt{2.2} & \textbf{24.3}~\dlt{5.9} \\
\bottomrule
\end{tabular}
\end{subtable}
\hfill
\begin{subtable}[t]{0.48\textwidth}
\centering
\scriptsize
\setlength{\tabcolsep}{3.5pt}
\caption{Attack that gains most from \emph{+aug}.}
\label{tab:aug-best}
\begin{tabular}{lccc}
\toprule
Family & IAR & Diffusion & Language \\
Model & RAR-XXL & DiT-MoE-XL & Pythia-6.9B \\
\midrule
Attack & Min-K\%++ & SecMI & CAMIA \\
Baseline & 13.0 & 2.0 & 11.5 \\
\;\;+aug & 16.4~\dlt{3.4} & 2.8~\dlt{0.8} & 15.5~\dlt{4.0} \\
\;\;+ours & \textbf{17.8}~\dlt{4.8} & \textbf{17.9}~\dlt{15.9} & \textbf{23.3}~\dlt{11.8} \\
\bottomrule
\end{tabular}
\end{subtable}
\end{table}

\Cref{tab:iar_mia_tpr} reports results on IARs. Trajectory features lift TPR in 23 of the 24 attack--model cells, by a median factor of $1.5\times$ (e.g.\ Min-K\%++ on VAR-d30 from $9.5$ to $15.5$). The improvement persists even against the strongest available baseline: on RAR-XXL, SOTA improves from $44.4$ to $47.3$. The effect is largest on diffusion models (\Cref{tab:diffusion_mia_tpr}), where several base attacks are close to uninformative and the trajectory supplies most of the signal. On DiT-MoE-XL, all baseline attacks rise significantly, e.g. Denoising Loss TPR$@1\%$FPR rises from $1.4$ to $26.7$. On the other hand, \Cref{tab:llm_mia_domain_pythia_6_9b} reports TPR on Pythia-6.9B across the seven MIMIR domains. Averaged over domains, every attack improves, and the spread between attacks narrows sharply. All attacks converge to a similar fused value within each domain, as if the trajectory supplied a common signal the base scores capture only partially. 

Corresponding AUC values (\Cref{tab:diffusion_mia_auc,tab:iar_mia_auc,tab:llm_mia_average_auc} in Appendix~\ref{app:mia-results}) move by points rather than factors, matching expectations for a method acting on the low-FPR tail. Dataset inference follows suit (Appendix~\ref{app:di-results}). Using trajectory features reduces the samples required to reach $p \leq 0.01$ in 61 of 86 settings. The savings peak exactly where baselines struggle most, driving several settings from 300--500 samples down to under 150.

\paragraph{Recursion is not simply more queries.}
The natural objection to any multi-query attack is that the gains come from the query budget rather than from the structure of the queries. The matched-compute control tests this directly: we generate $T$ independent perturbations of $Z_0$, run the base attack on each, and aggregate before scoring. This does not recover the gains. It usually performs at or below the \singlequery baseline, and on several attacks it degrades performance substantially (e.g.\ SURP on Pythia-6.9B falls from $18.4$ to $10.2$, \Cref{tab:aug-pair}). The direction of this result is what the anchoring account predicts. Membership evidence is concentrated at the exact sample: the model's response is sharp at $Z_0$ and unremarkable after even light augmentation, since the perturbed copies were themselves never trained on. Averaging over a neighborhood therefore integrates the signal away rather than denoising it. Recursion avoids this because the auditor does not choose where the trajectory goes. The model does, and if a memorized point exists it is pulled toward it.

\paragraph{Extended evaluations and ablations.} The pattern generalizes broadly. \Cref{tab:llm_mia_average_tpr,tab:llm_mia_average_auc} extend our evaluation to four additional LLMs -- including Mamba-2.8B -- confirming the behavior holds in non-Transformer state-space models. Additionally, \Cref{sec:model-scale} ablates scale by tracking the membership gap across eight vision models, spanning three architectural families and an order of magnitude in size.

\section{Conclusions}
Recursive self-generation exposes a fundamental asymmetry in how generative models process data. Over multiple iterations, unseen samples drift toward the broader training manifold, while seen samples remain anchored. This trajectory gap is a persistent phenomenon across language, diffusion, and autoregressive vision models. By repurposing this recursive behavior into an inference-time measurement, we extract a continuous membership signal without the cost of shadow models. Fusing trajectory features into existing attacks improves 118 of 122 evaluated configurations: a model's behavior under recursion carries evidence missed by a single query.

\section*{AI usage statement}
Large language models (LLMs) were used as an assistive tool to improve the readability, copy-editing, and structural flow of this manuscript, including LaTeX formatting and caption drafting. We also utilized LLMs as an interactive aid to refine our hypotheses, integrate attack code for baselines, improve plots clarity, fix mathematical notation errors and discuss results (the full-length manuscript were also checked for possible errors and inconsistencies).

The core intellectual contributions remain entirely our own: generative AI was not used to propose hypotheses, design conceptual framework or process datasets. All empirical results and numerical values reported were generated exclusively by our own experimental code. The authors thoroughly reviewed all AI-assisted work, verified all citations and technical statements, and assume full responsibility for the final manuscript.

\section*{Reproducibility statement}
To ensure the full reproducibility of our findings, all experiments in this work exclusively utilize open-source models, and publicly available, identically distributed datasets. We conduct all evaluations under strict, deterministic settings -- enforcing greedy decoding and standardized sequence lengths -- to eliminate confounding variance and isolate genuine membership signals. All experiments ran on a single machine with an NVIDIA A100 GPU and an AMD EPYC 7742 64-core CPU. A complete codebase -- containing our modular evaluation framework, data processing pipelines, and exact implementations of our method alongside all baseline attacks will be made publicly available upon the acceptance of this manuscript.

\section*{Impact Statement}
This work advances methods for auditing generative models by improving membership and dataset inference through recursive self-generation (\method). The primary positive impact is stronger accountability: \method can help detect memorization of sensitive, proprietary, or benchmark data, supporting privacy audits, copyright verification, and unlearning validation across model families and modalities.

While enhanced inference capabilities can assist in model auditing and transparency, they also require responsible application to avoid potential misuse. We frame \method as a tool for research evaluation, compliance monitoring, and internal red-teaming. It is important to note that our method provides statistical evidence rather than a definitive proof of data inclusion; therefore, results should be interpreted alongside additional forensic and procedural evidence within a broader data governance framework.

\subsubsection*{Acknowledgments}
We gratefully acknowledge Polish high-performance computing infrastructure PLGrid for providing computer facilities and support within computational grant no. PLG/2025/018391. This research was partially funded by National Science Centre, Poland, grant no: 2023/51/I/ST6/02854.

\bibliography{iclr2027_conference}

@inproceedings{shokri2017mia,
  title={Membership inference attacks against machine learning models},
  author={Shokri, Reza and Stronati, Marco and Song, Congzheng and Shmatikov, Vitaly},
  booktitle={2017 IEEE symposium on security and privacy (SP)},
  pages={3--18},
  year={2017},
  organization={IEEE}
}

@inproceedings{
    maini2021di,
    title={Dataset Inference: Ownership Resolution in Machine Learning},
    author={Pratyush Maini and Mohammad Yaghini and Nicolas Papernot},
    booktitle={International Conference on Learning Representations},
    year={2021},
    url={https://openreview.net/forum?id=hvdKKV2yt7T}
}

@article{singh2024evaluation,
  title={Evaluation data contamination in LLMs: how do we measure it and (when) does it matter?},
  author={Singh, Aaditya K and Kocyigit, Muhammed Yusuf and Poulton, Andrew and Esiobu, David and Lomeli, Maria and Szilvasy, Gergely and Hupkes, Dieuwke},
  journal={arXiv preprint arXiv:2411.03923},
  year={2024}
}

@inproceedings{alemohammad2023mad,
  title={Self-consuming generative models go mad},
  author={Alemohammad, Sina and Casco-Rodriguez, Josue and Luzi, Lorenzo and Humayun, Ahmed Imtiaz and Babaei, Hossein and LeJeune, Daniel and Siahkoohi, Ali and Baraniuk, Richard},
  booktitle={The Twelfth International Conference on Learning Representations},
  year={2023}
}

@article{shumailov2024mad2,
    author = {Shumailov, Ilia and Shumaylov, Zakhar and Zhao, Yiren and Papernot, Nicolas and Anderson, Ross and Gal, Yarin},
    year = {2024},
    month = {07},
    pages = {755-759},
    title = {AI models collapse when trained on recursively generated data},
    volume = {631},
    journal = {Nature},
    doi = {10.1038/s41586-024-07566-y}
}

@inproceedings{
li2024towards,
    title={Towards Black-Box Membership Inference Attack for Diffusion Models},
    author={Jingwei Li and Jing Dong and Tianxing He and Jingzhao Zhang},
    booktitle={Forty-second International Conference on Machine Learning},
    year={2025},
    url={https://openreview.net/forum?id=UlHHz2XYP7}
}

@inproceedings{burg2021memorization1,
 author = {van den Burg, Gerrit and Williams, Chris},
 booktitle = {Advances in Neural Information Processing Systems},
 editor = {M. Ranzato and A. Beygelzimer and Y. Dauphin and P.S. Liang and J. Wortman Vaughan},
 pages = {27916--27928},
 publisher = {Curran Associates, Inc.},
 title = {On Memorization in Probabilistic Deep Generative Models},
 url = {https://proceedings.neurips.cc/paper_files/paper/2021/file/eae15aabaa768ae4a5993a8a4f4fa6e4-Paper.pdf},
 volume = {34},
 year = {2021}
}

@inproceedings{bai2021memorization2,
    author = {Bai, Ching-Yuan and Lin, Hsuan-Tien and Raffel, Colin and Kan, Wendy Chi-wen},
    title = {On Training Sample Memorization: Lessons from Benchmarking Generative Modeling with a Large-scale Competition},
    year = {2021},
    isbn = {9781450383325},
    publisher = {Association for Computing Machinery},
    address = {New York, NY, USA},
    url = {https://doi.org/10.1145/3447548.3467198},
    doi = {10.1145/3447548.3467198},
    booktitle = {Proceedings of the 27th ACM SIGKDD Conference on Knowledge Discovery \& Data Mining},
    pages = {2534–2542},
    numpages = {9},
    location = {Virtual Event, Singapore},
    series = {KDD '21}
}

@inproceedings{ren2024memorization3,
  title={Unveiling and mitigating memorization in text-to-image diffusion models through cross attention},
  author={Ren, Jie and Li, Yaxin and Zeng, Shenglai and Xu, Han and Lyu, Lingjuan and Xing, Yue and Tang, Jiliang},
  booktitle={European Conference on Computer Vision},
  pages={340--356},
  year={2024},
  organization={Springer}
}

@article{
    gu2023memorization4,
    title={On Memorization in Diffusion Models},
    author={Xiangming Gu and Chao Du and Tianyu Pang and Chongxuan Li and Min Lin and Ye Wang},
    journal={Transactions on Machine Learning Research},
    issn={2835-8856},
    year={2025},
    url={https://openreview.net/forum?id=D3DBqvSDbj},
    note={}
}

@article{hans2024memorization5,
  title={Be like a goldfish, don't memorize! mitigating memorization in generative llms},
  author={Hans, Abhimanyu and Wen, Yuxin and Jain, Neel and Kirchenbauer, John and Kazemi, Hamid and Singhania, Prajwal and Singh, Siddharth and Somepalli, Gowthami and Geiping, Jonas and Bhatele, Abhinav and others},
  journal={Advances in Neural Information Processing Systems},
  volume={37},
  pages={24022--24045},
  year={2024}
}

@inproceedings{
    wen2024memorization6,
    title={Detecting, Explaining, and Mitigating Memorization in Diffusion Models},
    author={Yuxin Wen and Yuchen Liu and Chen Chen and Lingjuan Lyu},
    booktitle={The Twelfth International Conference on Learning Representations},
    year={2024},
    url={https://openreview.net/forum?id=84n3UwkH7b}
}

@article{Chen2025memorization7,
  title={Enhancing Privacy-Utility Trade-offs to Mitigate Memorization in Diffusion Models},
  author={Chen Chen and Daochang Liu and Mubarak Shah and Chang Xu},
  journal={2025 IEEE/CVF Conference on Computer Vision and Pattern Recognition (CVPR)},
  year={2025},
  pages={8182-8191},
  url={https://api.semanticscholar.org/CorpusID:278129333}
}

@inproceedings{ye2022enhanced,
  title={Enhanced membership inference attacks against machine learning models},
  author={Ye, Jiayuan and Maddi, Aadyaa and Murakonda, Sasi Kumar and Bindschaedler, Vincent and Shokri, Reza},
  booktitle={Proceedings of the 2022 ACM SIGSAC conference on computer and communications security},
  pages={3093--3106},
  year={2022}
}

@inproceedings{carlini2021zlib,
  title={Extracting training data from large language models},
  author={Carlini, Nicholas and Tramer, Florian and Wallace, Eric and Jagielski, Matthew and Herbert-Voss, Ariel and Lee, Katherine and Roberts, Adam and Brown, Tom and Song, Dawn and Erlingsson, Ulfar and others},
  booktitle={30th USENIX security symposium (USENIX Security 21)},
  pages={2633--2650},
  year={2021}
}

@inproceedings{carlini2022membership,
  title={Membership inference attacks from first principles},
  author={Carlini, Nicholas and Chien, Steve and Nasr, Milad and Song, Shuang and Terzis, Andreas and Tramer, Florian},
  booktitle={2022 IEEE symposium on security and privacy (SP)},
  pages={1897--1914},
  year={2022},
  organization={IEEE}
}

@article{maini2024did,
  author = {Maini, Pratyush and Jia, Hengrui and Papernot, Nicolas and Dziedzic, Adam},
  title = {{LLM} Dataset Inference: Did you train on my dataset?},
  journal = {CoRR},
  volume = {abs/2406.06443},
  year = {2024},
  doi = {10.48550/arXiv.2406.06443},
  eprint = {2406.06443},
  archivePrefix = {arXiv},
  url = {https://doi.org/10.48550/arXiv.2406.06443}
}

@inproceedings{dubinski2025cdi,
  title={Cdi: Copyrighted data identification in diffusion models},
  author={Dubi{\'n}ski, Jan and Kowalczuk, Antoni and Boenisch, Franziska and Dziedzic, Adam},
  booktitle={Proceedings of the Computer Vision and Pattern Recognition Conference},
  pages={18674--18684},
  year={2025}
}

@inproceedings{chang2025camia,
  title={Context-aware membership inference attacks against pre-trained large language models},
  author={Chang, Hongyan and Shamsabadi, Ali Shahin and Katevas, Kleomenis and Haddadi, Hamed and Shokri, Reza},
  booktitle={Proceedings of the 2025 Conference on Empirical Methods in Natural Language Processing},
  pages={7288–7310},
  year={2025}
}

@inproceedings{shi2023mink,
 author = {Shi, Weijia and Ajith, Anirudh and Xia, Mengzhou and Huang, Yangsibo and Liu, Daogao and Blevins, Terra and Chen, Danqi and Zettlemoyer, Luke},
 booktitle = {International Conference on Learning Representations},
 editor = {B. Kim and Y. Yue and S. Chaudhuri and K. Fragkiadaki and M. Khan and Y. Sun},
 pages = {51826--51843},
 title = {Detecting Pretraining Data from Large Language Models},
 url = {https://proceedings.iclr.cc/paper_files/paper/2024/file/e32ad85fa27be4a9868d55703f01323e-Paper-Conference.pdf},
 volume = {2024},
 year = {2024}
}

@inproceedings{carlini2023extracting,
  title={Extracting training data from diffusion models},
  author={Carlini, Nicholas and Hayes, Jamie and Nasr, Milad and Jagielski, Matthew and Sehwag, Vikash and Tramer, Florian and Balle, Borja and Ippolito, Daphne and Wallace, Eric},
  booktitle={32nd USENIX security symposium (USENIX Security 23)},
  pages={5253--5270},
  year={2023}
}

@inproceedings{xie2024recall,
  title={Recall: Membership inference via relative conditional log-likelihoods},
  author={Xie, Roy and Wang, Junlin and Huang, Ruomin and Zhang, Minxing and Ge, Rong and Pei, Jian and Gong, Neil Zhenqiang and Dhingra, Bhuwan},
  booktitle={Proceedings of the 2024 Conference on Empirical Methods in Natural Language Processing},
  pages={8671--8689},
  year={2024}
}

@inproceedings{duan2023secmi,
  title={Are diffusion models vulnerable to membership inference attacks?},
  author={Duan, Jinhao and Kong, Fei and Wang, Shiqi and Shi, Xiaoshuang and Xu, Kaidi},
  booktitle={International Conference on Machine Learning},
  pages={8717--8730},
  year={2023},
  organization={PMLR}
}

@article{bertran2023hinge,
  title={Scalable membership inference attacks via quantile regression},
  author={Bertran, Martin and Tang, Shuai and Roth, Aaron and Kearns, Michael and Morgenstern, Jamie H and Wu, Steven},
  journal={Advances in Neural Information Processing Systems},
  volume={36},
  pages={314--330},
  year={2023}
}

@inproceedings{yeom2018loss,
    title={Privacy risk in machine learning: Analyzing the connection to overfitting},
    author={Yeom, Samuel and Giacomelli, Irene and Fredrikson, Matt and Jha, Somesh},
    booktitle={2018 IEEE 31st computer security foundations symposium (CSF)},
    pages={268--282},
    year={2018},
    organization={IEEE}
}

@inproceedings{zhang2024minkpp,
    author = {Zhang, Jingyang and Sun, Jingwei and Yeats, Eric and Ouyang, Yang and Kuo, Martin and Zhang, Jianyi and Yang, Hao and Li, Hai},
    booktitle = {International Conference on Learning Representations},
    editor = {Y. Yue and A. Garg and N. Peng and F. Sha and R. Yu},
    pages = {64845--64862},
    title = {Min-K\%++: Improved Baseline for Pre-Training Data Detection from Large Language Models},
    url = {https://proceedings.iclr.cc/paper_files/paper/2025/file/a2e3b4132ab2e0b7a21e6e75da7f91a9-Paper-Conference.pdf},
    volume = {2025},
    year = {2025}
}

@inproceedings{choquette2021labelonlymia,
  title={Label-only membership inference attacks},
  author={Choquette-Choo, Christopher A and Tramer, Florian and Carlini, Nicholas and Papernot, Nicolas},
  booktitle={International conference on machine learning},
  pages={1964--1974},
  year={2021},
  organization={PMLR}
}

@inproceedings{wu2024yoqo,
    title={You Only Query Once: An Efficient Label-Only Membership Inference Attack},
    author={YUTONG WU and Han Qiu and Shangwei Guo and Jiwei Li and Tianwei Zhang},
    booktitle={The Twelfth International Conference on Learning Representations},
    year={2024},
    url={https://openreview.net/forum?id=7WsivwyHrS}
}

@article{zhang2022healthmia,
  title={Membership inference attacks against synthetic health data},
  author={Zhang, Ziqi and Yan, Chao and Malin, Bradley A},
  journal={Journal of biomedical informatics},
  volume={125},
  pages={103977},
  year={2022},
  publisher={Elsevier}
}

@article{tao2025informia,
  author = {Tao, Jiashu and Shokri, Reza},
  title = {(Token-Level) {InfoRMIA}: Stronger Membership Inference and Memorization Assessment for {LLM}s},
  journal = {CoRR},
  volume = {abs/2510.05582},
  year = {2025},
  doi = {10.48550/arXiv.2510.05582},
  eprint = {2510.05582},
  archivePrefix = {arXiv},
  url = {https://doi.org/10.48550/arXiv.2510.05582}
}

@article{wang2004ssim,
  title={Image quality assessment: from error visibility to structural similarity},
  author={Wang, Zhou and Bovik, Alan C and Sheikh, Hamid R and Simoncelli, Eero P},
  journal={IEEE transactions on image processing},
  volume={13},
  number={4},
  pages={600--612},
  year={2004},
  publisher={IEEE}
}

@inproceedings{zhang2018lpips,
  title={The unreasonable effectiveness of deep features as a perceptual metric},
  author={Zhang, Richard and Isola, Phillip and Efros, Alexei A and Shechtman, Eli and Wang, Oliver},
  booktitle={Proceedings of the IEEE conference on computer vision and pattern recognition},
  pages={586--595},
  year={2018}
}

@inproceedings{
    zawalski2026codec,
    title={Detecting Data Contamination in {LLM}s via In-Context Learning},
    author={Micha{\l} Zawalski and Meriem Boubdir and Klaudia Ba{\l}azy and Besmira Nushi and Pablo Ribalta},
    booktitle={The Fourteenth International Conference on Learning Representations},
    year={2026},
    url={https://openreview.net/forum?id=YlpaaYxx4t}
}

@article{kong2023pia,
  title={An efficient membership inference attack for the diffusion model by proximal initialization},
  author={Kong, Fei and Duan, Jinhao and Ma, RuiPeng and Shen, Hengtao and Zhu, Xiaofeng and Shi, Xiaoshuang and Xu, Kaidi},
  journal={arXiv preprint arXiv:2305.18355},
  year={2023}
}

@inproceedings{maini2025reassessing,
  title={Reassessing EMNLP 2024’s Best Paper: Does Divergence-Based Calibration for MIAs Hold Up?},
  author={Maini, Pratyush and Suri, Anshuman},
  booktitle={The Fourth Blogpost Track at ICLR 2025}
}

@inproceedings{kowalczuk2025privacy,
    title={Privacy Attacks on Image AutoRegressive Models},
    author={Antoni Kowalczuk and Jan Dubi{\'n}ski and Franziska Boenisch and Adam Dziedzic},
    booktitle={Forty-second International Conference on Machine Learning},
    year={2025},
    url={https://openreview.net/forum?id=7SXXczJCWP}
}

@inproceedings{tian2024var,
 author = {Tian, Keyu and Jiang, Yi and Yuan, Zehuan and Peng, Bingyue and Wang, Liwei},
 booktitle = {Advances in Neural Information Processing Systems},
 doi = {10.52202/079017-2694},
 editor = {A. Globerson and L. Mackey and D. Belgrave and A. Fan and U. Paquet and J. Tomczak and C. Zhang},
 pages = {84839--84865},
 publisher = {Curran Associates, Inc.},
 title = {Visual Autoregressive Modeling: Scalable Image Generation via Next-Scale Prediction},
 url = {https://proceedings.neurips.cc/paper_files/paper/2024/file/9a24e284b187f662681440ba15c416fb-Paper-Conference.pdf},
 volume = {37},
 year = {2024}
}

@InProceedings{yu2024rar,
    author    = {Yu, Qihang and He, Ju and Deng, Xueqing and Shen, Xiaohui and Chen, Liang-Chieh},
    title     = {Randomized Autoregressive Visual Generation},
    booktitle = {Proceedings of the IEEE/CVF International Conference on Computer Vision (ICCV)},
    month     = {October},
    year      = {2025},
    pages     = {18431-18441}
}

@misc{fei2024dit,
      title={Scaling Diffusion Transformers to 16 Billion Parameters}, 
      author={Zhengcong Fei and Mingyuan Fan and Changqian Yu and Debang Li and Junshi Huang},
      year={2024},
      eprint={2407.11633},
      archivePrefix={arXiv},
      primaryClass={cs.CV},
      url={https://arxiv.org/abs/2407.11633}, 
}

@inproceedings{bao2022uvit,
  title={All are Worth Words: A ViT Backbone for Diffusion Models},
  author={Bao, Fan and Nie, Shen and Xue, Kaiwen and Cao, Yue and Li, Chongxuan and Su, Hang and Zhu, Jun},
  booktitle = {CVPR},
  year={2023}
}

@inproceedings{biderman2023pythia,
author = {Biderman, Stella and Schoelkopf, Hailey and Anthony, Quentin and Bradley, Herbie and O'Brien, Kyle and Hallahan, Eric and Khan, Mohammad Aflah and Purohit, Shivanshu and Prashanth, USVSN Sai and Raff, Edward and Skowron, Aviya and Sutawika, Lintang and Van Der Wal, Oskar},
title = {Pythia: a suite for analyzing large language models across training and scaling},
year = {2023},
publisher = {JMLR.org},
booktitle = {Proceedings of the 40th International Conference on Machine Learning},
articleno = {102},
numpages = {34},
location = {, Honolulu, Hawaii, USA, },
series = {ICML'23}
}

@software{black2021gptneo,
  author       = {Black, Sid and
                  Leo, Gao and
                  Wang, Phil and
                  Leahy, Connor and
                  Biderman, Stella},
  title        = {{GPT-Neo: Large Scale Autoregressive Language 
                   Modeling with Mesh-Tensorflow}},
  month        = mar,
  year         = 2021,
  note         = {{If you use this software, please cite it using 
                   these metadata.}},
  publisher    = {Zenodo},
  version      = {1.0},
  doi          = {10.5281/zenodo.5297715},
  url          = {https://doi.org/10.5281/zenodo.5297715}
}

@misc{wang2021gptj,
  author = {Wang, Ben and Komatsuzaki, Aran},
  title = {{GPT-J-6B: A 6 Billion Parameter Autoregressive Language Model}},
  howpublished = {\url{https://github.com/kingoflolz/mesh-transformer-jax}},
  year = 2021,
  month = May
}

@article{gu2023mamba,
  title={Mamba: Linear-time sequence modeling with selective state spaces},
  author={Gu, Albert and Dao, Tri},
  journal={arXiv preprint arXiv:2312.00752},
  year={2023}
}

@inproceedings{black2022gptneox,
  title={Gpt-neox-20b: An open-source autoregressive language model},
  author={Black, Sidney and Biderman, Stella and Hallahan, Eric and Anthony, Quentin and Gao, Leo and Golding, Laurence and He, Horace and Leahy, Connor and McDonell, Kyle and Phang, Jason and others},
  booktitle={Proceedings of BigScience Episode\# 5--Workshop on Challenges \& Perspectives in Creating Large Language Models},
  pages={95--136},
  year={2022}
}

@article{kynkaanniemi2019improved,
  title={Improved precision and recall metric for assessing generative models},
  author={Kynk{\"a}{\"a}nniemi, Tuomas and Karras, Tero and Laine, Samuli and Lehtinen, Jaakko and Aila, Timo},
  journal={Advances in neural information processing systems},
  volume={32},
  year={2019}
}

@article{heusel2017fid,
  title={Gans trained by a two time-scale update rule converge to a local nash equilibrium},
  author={Heusel, Martin and Ramsauer, Hubert and Unterthiner, Thomas and Nessler, Bernhard and Hochreiter, Sepp},
  journal={Advances in neural information processing systems},
  volume={30},
  year={2017}
}

@inproceedings{deng2009imagenet,
  title={Imagenet: A large-scale hierarchical image database},
  author={Deng, Jia and Dong, Wei and Socher, Richard and Li, Li-Jia and Li, Kai and Fei-Fei, Li},
  booktitle={2009 IEEE conference on computer vision and pattern recognition},
  pages={248--255},
  year={2009},
  organization={Ieee}
}

@inproceedings{lin2014microsoft,
  title={Microsoft coco: Common objects in context},
  author={Lin, Tsung-Yi and Maire, Michael and Belongie, Serge and Hays, James and Perona, Pietro and Ramanan, Deva and Doll{\'a}r, Piotr and Zitnick, C Lawrence},
  booktitle={European conference on computer vision},
  pages={740--755},
  year={2014},
  organization={Springer}
}

@inproceedings{duan2024mimir,
      title={Do Membership Inference Attacks Work on Large Language Models?}, 
      author={Michael Duan and Anshuman Suri and Niloofar Mireshghallah and Sewon Min and Weijia Shi and Luke Zettlemoyer and Yulia Tsvetkov and Yejin Choi and David Evans and Hannaneh Hajishirzi},
      year={2024},
      booktitle={Conference on Language Modeling (COLM)},
}

@article{zhang2024adaptive,
  title={Adaptive pre-training data detection for large language models via surprising tokens},
  author={Zhang, Anqi and Wu, Chaofeng},
  journal={arXiv preprint arXiv:2407.21248},
  year={2024}
}
\bibliographystyle{iclr2027_conference}

\appendix
\clearpage
\section{Limitations}
We evaluate three model families; recursion applies wherever a model maps samples to samples of the same type, and extending it to audio, video, and diffusion language models is left to future work. While recursion incurs a linear overhead of $T$ forward passes per sample at inference time, this cost is orders of magnitude lower than the training requirements of shadow model baselines. Furthermore, we leave the optimization of this overhead -- such as implementing adaptive early-stopping criteria for highly confident trajectories to future work.

\section{Extended Related Works}
\label{app:related}
Our work builds upon three intersecting lines of prior research: the characterization of data memorization in generative models, the evolution of membership inference, and the dynamics of model collapse during recursive generation.

\subsection{Memorization}
Memorization -- the tendency of generative models to reproduce training examples rather than generate novel samples -- has been studied across multiple model families and from both measurement and mitigation perspectives. \citet{burg2021memorization1} formalized the problem for probabilistic generative models such as VAEs, showing that memorization differs fundamentally from mode collapse and overfitting and is not captured by commonly-used nearest-neighbor tests. For diffusion models, \citet{gu2023memorization4} show that the denoising score matching objective has a closed-form optimum that can only replicate training samples, and introduces the EMM metric to quantify how dataset size and model configuration govern the generalization-to-memorization transition. \citet{ren2024memorization3} localize this phenomenon through bright-ending cross-attention patterns, while the sharpness-based framework of~\citep{wen2024memorization6} justifies score-difference memorization metrics and proposes mitigation via sharpness-aware regularization of the initial noise. The benchmarking study~\citep{bai2021memorization2} demonstrated that standard evaluation metrics fail to surface memorization even in competitive settings. Mitigation has been tackled both for LLMs, where \citet{hans2024memorization5} propose the goldfish loss that excludes randomly sampled token subsets from the training objective to prevent verbatim reproduction, and for text-to-image diffusion models, where \citet{Chen2025memorization7} address the privacy–utility tension by combining prompt re-anchoring with semantic prompt search to improve both dimensions simultaneously.

\subsection{Membership and Dataset Inference}
A second line of work investigates whether specific examples or datasets can be identified from model behavior. Because individual Membership Inference Attacks (MIAs) can be confounded by distribution shifts~\citep{maini2024did}, recent literature often favors Dataset Inference (DI), which aggregates feature evidence across many samples to statistically detect training data usage~\citep{maini2025reassessing,dubinski2025cdi,kowalczuk2025privacy}. Concurrently, individual MIA methods must adapt to increasingly restrictive black-box deployments. Furthermore, approaches based on training multiple shadow models to learn membership distributions~\citep{ye2022enhanced,carlini2021zlib} are now computationally infeasible for massive modern architectures. Consequently, modern attacks must extract signals using only limited outputs rather than internal weights or gradients~\citep{zhang2024minkpp,chang2025camia,tao2025informia}.

In these restricted settings, recent black-box attacks heavily rely on output variations. For example, \citet{li2024towards} perform MIAs on diffusion models by repeatedly perturbing a target image via an API, averaging the results, and comparing them to the original sample. However, in an interrogation analogy, this approach merely asks multiple paraphrased versions of the exact same question. Because the target sample is perturbed each time independently, the query does not dynamically evolve in response to the model's previous answers, leaving deeper structural memorization unexploited.

\subsection{Model Collapse}
The last, but very important point is the literature on recursive self-training in generative models. \citet{alemohammad2023mad} showed that self-consuming generative loops lead to progressive degradation in quality or diversity when insufficient fresh real data is injected at each generation, a phenomenon they term Model Autophagy Disorder. Their analysis is especially important for our setting because it frames repeated regeneration not as a neutral operation, but as a process that can magnify latent properties of the learned distribution. Closely related, \citet{shumailov2024mad2} showed that recursively training on model-generated data causes model collapse, where tails of the original distribution disappear and learned behavior drifts toward degenerate approximations. Taken together, these works suggest that iterative generation is structurally revealing: under repeated reuse, memorized or high-density regions may persist differently from non-member examples, while generic outputs may drift or collapse. Our method turns this insight into a privacy-auditing mechanism: rather than studying recursive generation as a training-time pathology, \textbf{we exploit recursive self-generation at inference time} to improve the membership-relevant differences.

\section{Model Details}
\label{app:model_details}
In our experiments, we consider two vision model families: image autoregressive models (IARs) and diffusion models. The IAR category includes VAR~\cite{tian2024var} and RAR~\cite{yu2024rar} variants, while the diffusion category includes DiT-MoE~\cite{fei2024dit} and UViT-T2I~\cite{bao2022uvit}. Furthermore, we evaluate five large language models (LLMs): four Transformer-based -- Pythia~\citep{biderman2023pythia}, GPT-NeoX~\citep{black2022gptneox}, GPT-Neo~\citep{black2021gptneo} and GPT-J~\citep{wang2021gptj} -- and one state-space model (SSM), Mamba~\citep{gu2023mamba}. Across all settings, we focus on representative, high-performing model variants.

\begin{table*}[h!]
    \centering
    \caption{\textbf{Vision models used in evaluation.}}
    \label{tab:vision_model_details}
    \tiny
    \begin{tabular}{lccccccccc}
        \toprule
        & \multicolumn{6}{c}{\textbf{IAR Models}} & \multicolumn{3}{c}{\textbf{Diffusion Models}} \\
        \cmidrule(r){2-7} \cmidrule(r){8-10}
        & VAR-d30 & VAR-d24 & VAR-d20 & RAR-XXL & RAR-XL & RAR-L
        & DiT-MoE-G & DiT-MoE-XL & UViT-T2I-Deep \\
        \midrule
        \textbf{Model parameters} & 2.1B & 1.0B & 600M & 1.5B & 955M & 462M & 16.5B & 4.1B & 141M \\
        \textbf{Training epochs}   & 350 & 300 & 250 & 400 & 400 & 400  & — & — & — \\
        \textbf{FID}              & 1.92 & 2.33 & 2.95 & 1.48 & 1.50 & 1.70 & 1.72 & 2.10 & 5.48 \\
        \bottomrule
    \end{tabular}
\end{table*}

\begin{table*}[h]
\centering
\tiny
\caption{\textbf{Language models used in evaluation.}}
\label{tab:llm_model_details}
\begin{tabular}{l c c c c c c}
    \toprule
    & \textbf{Pythia} & \textbf{Pythia} & \textbf{GPT-NeoX} 
    & \textbf{GPT-Neo} & \textbf{GPT-J} & \textbf{Mamba} \\
    \midrule
    \textbf{Parameters}
    & 6.9B & 12B & 20B & 2.7B & 6B & 2.8B \\
    
    \textbf{Training tokens}
    & 300B & 300B & 472B & 420B & 402B & 300B \\
    \bottomrule
\end{tabular}
\end{table*}

\section{Dataset Details}
\label{app:dataset_details}
For vision models we use ImageNet~\cite{deng2009imagenet} and COCO~\cite{lin2014microsoft} datasets that have publicly known and available train/test splits. For language models, we use the established MIA benchmark Mimir~\cite{duan2024mimir} with the \textit{ngram\_7\_0.2} split.

\begin{table}[h!]
    \centering
    \caption{\textbf{Member and non-member sources, per model family.}}
    \label{tab:data_details}
    \tiny
    \begin{tabular}{lccccccccc}
        \toprule
        & \textbf{VAR} & \textbf{RAR} & \textbf{DiT-MoE} & \textbf{UViT-T2I}
        & \textbf{Pythia} & \textbf{GPT-NeoX} & \textbf{Mamba} & \textbf{GPT-Neo} & \textbf{GPT-J} \\
        \midrule
        \textbf{Dataset}
        & ImageNet
        & ImageNet
        & ImageNet
        & COCO
        & Mimir
        & Mimir
        & Mimir
        & Mimir
        & Mimir \\
        \bottomrule
    \end{tabular}
\end{table}

\section{Metrics Details}
\label{app:metrics}
The following metrics are computed over the sequence of model outputs collected across \method iterations, capturing how the model's generative behavior evolves under repeated generation.

\subsection{Features for Language Models}
\label{app:llm-features}
As in the vision case, every candidate document is run through a self-consuming,
sliding-window trajectory of $T$ iterations, and each raw signal below is recorded
once per iteration and reduced over the trajectory via its mean, final-iteration value, or
linear slope (noted per feature), except where the signal is already a single whole-trajectory
scalar.

\begin{description}
\item[Token Match Ratio and Verbatim Recall:]
At each iteration the model regenerates $n$ tokens from a fixed window; we compare the regenerated continuation $\hat{y}^{(t)}$ against the true continuation $y$ token-by-token. The match ratio is the fraction of positions reproduced exactly, reduced via its final value, its slope across the trajectory, and its leading \emph{depth} --- the run of iterations from the start of the trajectory for which the match ratio stays at or above $0.5$. Verbatim recall is the length of the longest exact-match streak observed anywhere in the trajectory, tracked at token granularity across iteration boundaries. A soft variant relaxes exact equality to ``the true token was among the model's own top-$k$ candidates,'' giving a soft-verbatim streak that is more forgiving of near-misses. A high, stable match ratio or a long exact-match streak indicates the model can reconstruct the true continuation from a truncated context, a strong sign of memorization.
\[
\mathrm{MatchRatio}_t = \frac{1}{n}\sum_{i=1}^{n} \mathbb{1}\!\left[\hat{y}_i^{(t)} = y_i\right],
\]
\[
\mathrm{Verbatim} = \max_{t}\;\ell_t, \quad \ell_t = \text{longest exact-match run ending at step } t.
\]
\item[Mean Reciprocal Rank (MRR):]
At each generation step, the true token $y_i$ is assigned a rank $r_i$ under the model's own predicted distribution (rank 1 = the model's top choice). MRR aggregates this rank into one confidence score per iteration; values close to 1 mean the ground-truth continuation was consistently the model's most likely prediction.
\[
\mathrm{MRR}_t = \frac{1}{n}\sum_{i=1}^{n} \frac{1}{r_i},
\qquad
r_i = \left|\{v \in \mathcal{V} : \ell_v > \ell_{y_i}\}\right| + 1
\]
where $\ell_v$ is the logit assigned to vocabulary token $v$ and $\mathcal{V}$ is the model's vocabulary. Reduced via its mean across the trajectory.
\item[Log-Probability:]
Complementary to the rank-based scores above, the mean log-probability the model assigns to the true continuation measures calibrated confidence rather than ordinal preference.
\[
\mathrm{LogProb}_t = \frac{1}{n}\sum_{i=1}^{n} \left(\ell_{y_i} - \log\!\!\sum_{v \in \mathcal{V}} e^{\ell_v}\right)
\]
Reduced via its mean across the trajectory.
\item[Predictive Entropy and Logit Margin:]
Beyond agreement with the ground truth, we track how confidently the model commits to its \emph{own}, greedily generated continuation at each step, via the entropy of its output distribution and the gap between its top-2 candidate logits. A model reproducing a memorized sample tends to generate with lower entropy and a wider top-1/top-2 margin than one improvising.
\[
H_t = -\sum_{v \in \mathcal{V}} p_v \log p_v,
\qquad
\mathrm{Margin}_t = \ell_{(1)} - \ell_{(2)}
\]
where $p_v = \mathrm{softmax}(\ell)_v$ and $\ell_{(1)} \geq \ell_{(2)}$ are the model's own top two logits at that step. Both reduced via their mean across the trajectory.
\item[Jaccard Similarity:]
A coarser, string-level similarity between the regenerated continuation and the ground truth, computed over word trigrams and anchored to the true continuation at the current window offset. Unlike the exact match ratio, it remains informative even when regeneration diverges token-by-token but preserves the same local phrasing.
\[
\mathrm{Jaccard}_t = \frac{|\mathcal{T}(\hat{y}^{(t)}) \cap \mathcal{T}(y)|}{|\mathcal{T}(\hat{y}^{(t)}) \cup \mathcal{T}(y)|}
\]
where $\mathcal{T}(\cdot)$ denotes the set of word trigrams in a text. Reduced via its mean across the trajectory.
\item[Edit Distance:]
The token-level Levenshtein distance between the regenerated and true continuations --- a distance-based sibling of the match-ratio family that additionally penalizes insertions and deletions, not just substitutions.
\[
\mathrm{EditDist}_t = \mathrm{Lev}\!\left(\hat{y}^{(t)}, y\right)
\]
Reduced via its mean across the trajectory.
\item[Jensen--Shannon Divergence (JSD):]
Rather than comparing to the ground truth, JSD measures how much the model's own output distribution drifts between \emph{consecutive} regeneration steps, comparing the mean vocabulary-probability profile at iteration $t$ to that at $t{-}1$ (not to the original prefix, which the sliding window eventually evicts). A stable, low-divergence trajectory suggests the model has settled into reproducing a fixed, memorized continuation rather than continuing to explore alternatives.
\[
\mathrm{JSD}(P_t, P_{t-1}) = \tfrac{1}{2}\,\mathrm{KL}(P_t \,\Vert\, M) + \tfrac{1}{2}\,\mathrm{KL}(P_{t-1} \,\Vert\, M),
\qquad M = \tfrac{1}{2}(P_t + P_{t-1})
\]
where $P_t$ is the mean softmax probability vector over the vocabulary at iteration $t$. Reduced by taking the mean across the trajectory.
\end{description}

\subsection{Features for Vision Models}
\label{app:vision-features}
\begin{description}
\item[Mean Squared Error (MSE):]
Measures the average pixel-level reconstruction error between the generated image at a given iteration and the original input. Lower MSE indicates that the model consistently reproduces fine-grained pixel details across iterations, which is a strong signal of memorization.
\[
\mathrm{MSE}(x, \hat{x}) = \frac{1}{N} \sum_{i=1}^{N} \left( x_i - \hat{x}_i \right)^2
\]
\item[Structural Similarity Index Measure (SSIM)~\citep{wang2004ssim}:]
Evaluates perceptual similarity between the reconstructed image $\hat{x}$ and the original $x$ by jointly comparing luminance, contrast, and structural information across local image patches. Unlike MSE, SSIM is sensitive to perceptual distortions that are meaningful to human observers, and its stability across iterations serves as a complementary signal to pixel-level metrics.
\[
\mathrm{SSIM}(x, \hat{x}) = \frac{(2\mu_x \mu_{\hat{x}} + c_1)(2\sigma_{x\hat{x}} + c_2)}{(\mu_x^2 + \mu_{\hat{x}}^2 + c_1)(\sigma_x^2 + \sigma_{\hat{x}}^2 + c_2)}
\]
where $\mu_x$, $\mu_{\hat{x}}$ are local means, $\sigma_x^2$, $\sigma_{\hat{x}}^2$ are local variances, $\sigma_{x\hat{x}}$ is the cross-covariance, and $c_1$, $c_2$ are stabilization constants.

\item[Learned Perceptual Image Patch Similarity (LPIPS)~\citep{zhang2018lpips}:]
Quantifies perceptual dissimilarity between $x$ and $\hat{x}$ using deep feature representations extracted from a pretrained network $\phi$. By operating in a learned feature space rather than pixel space, LPIPS captures high-level semantic and textural differences that are invisible to MSE or SSIM, making it particularly sensitive to cases where a model reproduces semantic content while varying low-level details.
\[
\mathrm{LPIPS}(x, \hat{x}) = \sum_{l} \frac{1}{H_l W_l} \sum_{h,w} \left\| w_l \odot \left( \phi_l(x)_{hw} - \phi_l(\hat{x})_{hw} \right) \right\|_2^2
\]
where $\phi_l$ denotes the feature map at layer $l$ of the pretrained network and $w_l$ are learned channel-wise weights.
\end{description}

\newpage
\section{Additional MIA Results}
\label{app:mia-results}
\begin{table}[h!]
\centering
\scriptsize
\caption{\textbf{MIA on diffusion models} (AUC).}
\label{tab:diffusion_mia_auc}
\begin{tabular}{lcccc}
\toprule
 & \multicolumn{2}{c}{DiT-MoE-XL} & \multicolumn{2}{c}{UViT-T2I-Deep} \\
\cmidrule(lr){2-3} \cmidrule(lr){4-5}
Attack & Baseline & +ours & Baseline & +ours \\
\midrule
SecMI & 65.9 & 90.1~\dlt{24.2} & 68.9 & 69.6~\dlt{0.7} \\
Denoising Loss & 52.0 & 91.7~\dlt{39.7} & 60.6 & 63.5~\dlt{2.9} \\
PIA & 63.9 & 92.4~\dlt{28.5} & 70.4 & 71.4~\dlt{1.0} \\
PIAN & 61.8 & 90.0~\dlt{28.2} & 51.3 & 54.2~\dlt{2.9} \\
CDI & 98.3 & 98.5~\dlt{0.2} & 74.3 & 74.7~\dlt{0.4} \\
\bottomrule
\end{tabular}
\end{table} 

\begin{table}[h!]
\centering
\scriptsize
\caption{\textbf{MIA on IARs} (AUC).}
\label{tab:iar_mia_auc}
\resizebox{\linewidth}{!}{%
\begin{tabular}{lcccccccc}
\toprule
 & \multicolumn{2}{c}{VAR-d30} & \multicolumn{2}{c}{VAR-d24} & \multicolumn{2}{c}{RAR-XXL} & \multicolumn{2}{c}{RAR-XL} \\
\cmidrule(lr){2-3} \cmidrule(lr){4-5} \cmidrule(lr){6-7} \cmidrule(lr){8-9}
Attack & Baseline & +ours & Baseline & +ours & Baseline & +ours & Baseline & +ours \\
\midrule
Loss & 74.4 & 80.4~\dlt{6.0} & 58.3 & 60.9~\dlt{2.6} & 73.7 & 76.5~\dlt{2.8} & 64.8 & 67.2~\dlt{2.4} \\
Hinge & 75.3 & 79.8~\dlt{4.5} & 61.9 & 63.5~\dlt{1.6} & 81.2 & 82.8~\dlt{1.6} & 70.7 & 72.2~\dlt{1.5} \\
Min-K\%++ & 84.8 & 86.5~\dlt{1.7} & 67.1 & 67.9~\dlt{0.8} & 84.5 & 86.7~\dlt{2.2} & 75.5 & 77.1~\dlt{1.6} \\
SURP & 74.4 & 80.3~\dlt{5.9} & 58.3 & 60.9~\dlt{2.6} & 76.7 & 80.0~\dlt{3.3} & 66.7 & 69.6~\dlt{2.9} \\
CAMIA & 79.8 & 81.5~\dlt{1.7} & 64.8 & 66.0~\dlt{1.2} & 66.0 & 71.2~\dlt{5.2} & 58.5 & 63.2~\dlt{4.7} \\
\midrule
SOTA & 99.6 & 99.6~\dlt{0.0} & 87.4 & 87.6~\dlt{0.2} & 97.3 & 97.4~\dlt{0.1} & 92.1 & 92.2~\dlt{0.1} \\
\bottomrule
\end{tabular}
}
\end{table}

\begin{table}[h!]
\centering
\scriptsize
\caption{\textbf{MIA on LLMs} (TPR@$1\%$FPR, averaged over MIMIR domains).}
\label{tab:llm_mia_average_tpr}
\resizebox{\linewidth}{!}{%
\begin{tabular}{lcccccccc}
\toprule
 & \multicolumn{2}{c}{Pythia-12B} & \multicolumn{2}{c}{GPT-J-6B} & \multicolumn{2}{c}{GPT-Neo-2.7B} & \multicolumn{2}{c}{Mamba-2.8B} \\
\cmidrule(lr){2-3} \cmidrule(lr){4-5} \cmidrule(lr){6-7} \cmidrule(lr){8-9}
Attack & Baseline & +ours & Baseline & +ours & Baseline & +ours & Baseline & +ours \\
\midrule
Loss & 12.9 & 22.7~\dlt{9.8} & 14.7 & 23.3~\dlt{8.6} & 14.6 & 24.5~\dlt{9.9} & 12.4 & 22.3~\dlt{9.9} \\
Zlib & 10.8 & 23.1~\dlt{12.3} & 10.5 & 22.8~\dlt{12.3} & 9.6 & 24.0~\dlt{14.4} & 11.1 & 22.6~\dlt{11.5} \\
Hinge & 8.7 & 19.6~\dlt{10.9} & 8.2 & 20.2~\dlt{12.0} & 7.9 & 21.3~\dlt{13.4} & 10.0 & 20.0~\dlt{10.0} \\
Recall & 12.2 & 23.8~\dlt{11.6} & 10.6 & 22.5~\dlt{11.9} & 10.5 & 23.8~\dlt{13.3} & 10.8 & 21.8~\dlt{11.0} \\
Min-K\% & 14.8 & 24.3~\dlt{9.5} & 15.6 & 23.5~\dlt{7.9} & 16.0 & 24.9~\dlt{8.9} & 13.6 & 22.5~\dlt{8.9} \\
Min-K\%++ & 11.5 & 22.2~\dlt{10.7} & 11.5 & 22.2~\dlt{10.7} & 8.6 & 22.3~\dlt{13.7} & 9.8 & 21.6~\dlt{11.8} \\
SURP & 14.6 & 24.0~\dlt{9.4} & 18.8 & 23.7~\dlt{4.9} & 16.3 & 24.7~\dlt{8.4} & 13.8 & 22.9~\dlt{9.1} \\
CAMIA & 11.2 & 21.6~\dlt{10.4} & 10.5 & 21.0~\dlt{10.5} & 10.7 & 21.8~\dlt{11.1} & 12.3 & 21.8~\dlt{9.5} \\
\bottomrule
\end{tabular}
}
\end{table}
\begin{table}[h!]
\centering
\scriptsize
\caption{\textbf{MIA on LLMs} (AUC, averaged over MIMIR domains).}
\label{tab:llm_mia_average_auc}
\resizebox{\linewidth}{!}{%
\begin{tabular}{lcccccccc}
\toprule
 & \multicolumn{2}{c}{Pythia-12B} & \multicolumn{2}{c}{GPT-J-6B} & \multicolumn{2}{c}{GPT-Neo-2.7B} & \multicolumn{2}{c}{Mamba-2.8B} \\
\cmidrule(lr){2-3} \cmidrule(lr){4-5} \cmidrule(lr){6-7} \cmidrule(lr){8-9}
Attack & Baseline & +ours & Baseline & +ours & Baseline & +ours & Baseline & +ours \\
\midrule
Loss & 71.7 & 68.1~\dlt{-3.6} & 71.8 & 70.1~\dlt{-1.7} & 71.1 & 69.8~\dlt{-1.3} & 71.2 & 68.6~\dlt{-2.6} \\
Zlib & 70.5 & 68.6~\dlt{-1.9} & 70.4 & 70.2~\dlt{-0.2} & 69.7 & 70.0~\dlt{0.3} & 70.1 & 68.9~\dlt{-1.2} \\
Hinge & 63.3 & 65.5~\dlt{2.2} & 64.0 & 67.0~\dlt{3.0} & 62.9 & 67.0~\dlt{4.1} & 63.4 & 66.3~\dlt{2.9} \\
Recall & 67.6 & 68.9~\dlt{1.3} & 67.0 & 69.0~\dlt{2.0} & 66.0 & 69.1~\dlt{3.1} & 67.2 & 67.7~\dlt{0.5} \\
Min-K\% & 72.3 & 69.7~\dlt{-2.6} & 72.2 & 71.8~\dlt{-0.4} & 71.3 & 71.2~\dlt{-0.1} & 71.6 & 69.6~\dlt{-2.0} \\
Min-K\%++ & 66.3 & 66.7~\dlt{0.4} & 66.7 & 68.8~\dlt{2.1} & 65.1 & 68.0~\dlt{2.9} & 64.8 & 66.5~\dlt{1.7} \\
SURP & 72.4 & 70.2~\dlt{-2.2} & 72.2 & 71.9~\dlt{-0.3} & 71.3 & 71.3~\dlt{0.0} & 71.6 & 70.2~\dlt{-1.4} \\
CAMIA & 67.4 & 68.0~\dlt{0.6} & 67.3 & 69.6~\dlt{2.3} & 66.7 & 69.0~\dlt{2.3} & 66.5 & 68.5~\dlt{2.0} \\
\bottomrule
\end{tabular}
}
\end{table}

\section{Dataset Inference Results}
\label{app:di-results}

\paragraph{DI pipeline.}
The DI pipeline is identical to MIAs, replacing the target label with dataset-origin variable $D$. The same $Z_t$, $\phi_t$, and trajectory-fusion machinery is used; only label semantics and calibration change. For DI, splitting/evaluation are performed at the dataset or source-group level, and per-sample logits are aggregated by a fixed mean rule into a dataset-level score. Both MIA and DI setups inherit standard generative privacy-audit conventions~\citep{kowalczuk2025privacy,dubinski2025cdi,zhang2024minkpp}.

\paragraph{Results.}
Dataset inference asks the same question at the level of a corpus rather than a sample, and the trajectory features stay unchanged. Only the label semantics
and the aggregation rule differ (\Cref{sec:mia-application}). The relevant measure is sample efficiency -- how many samples are needed before the test reaches significance. \Cref{fig:di-summary} compares the minimum $n$ required to reach $p \leq 0.01$ with and without trajectory features, over 86 attack--model settings. Trajectory features help in 61 of them, leave 23 unchanged and cost slightly in 2, and the advantage is largest precisely where the baseline is most expensive, with several settings requiring 300 -- 500 samples at baseline dropping to under 150.

\begin{figure}[h!]
  \centering
  \includegraphics[width=0.6\linewidth]{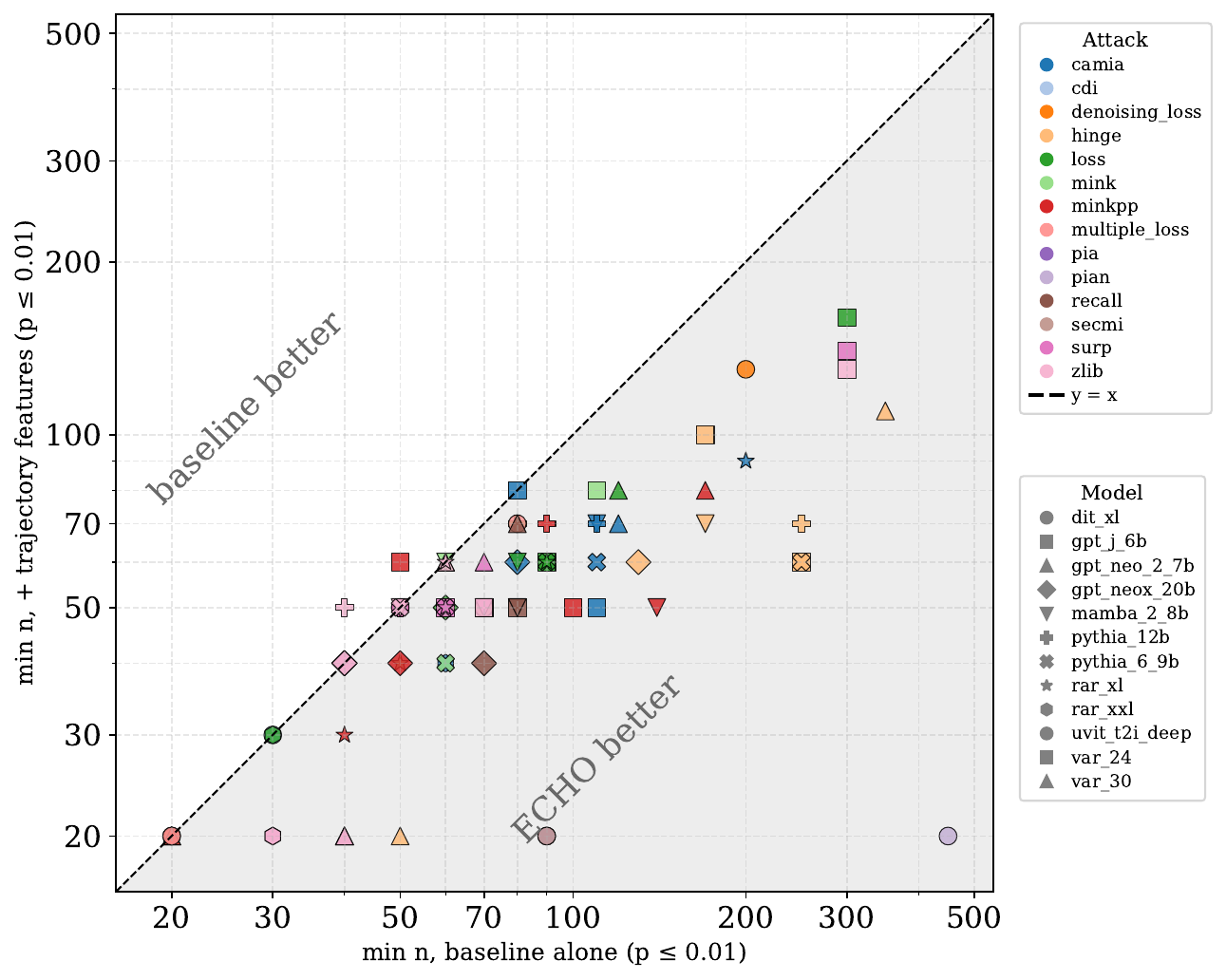}
  \caption{\textbf{Dataset inference reaches significance from fewer samples.} Each point is one attack on one model: the smallest $n$ at which the test attains $p\le0.01$, for the base attack alone against the same attack fused with trajectory features. Points below the diagonal are settings where the trajectory helps; 61 of 86 fall there, 23 sit on it, and 2 above. Color is the attack, shape describes the model.}
  \label{fig:di-summary}
\end{figure}

\section{Precision and Recall for Generative Models}
\label{app:rpc}
\Cref{fig:prec-rec-models} shows Generative Precision and Recall across iterations for VAR-d30 and DiT-MoE-XL. In both models and both metrics, members consistently score higher than non-members throughout all iterations, confirming that the membership signal is stable and model-agnostic. Notably, the gap between members and non-members widens as iterations progress, indicating that recursive self-generation amplifies the underlying asymmetry rather than merely preserving it.

\begin{figure}[h!]
  \centering
  \includegraphics[width=\linewidth]{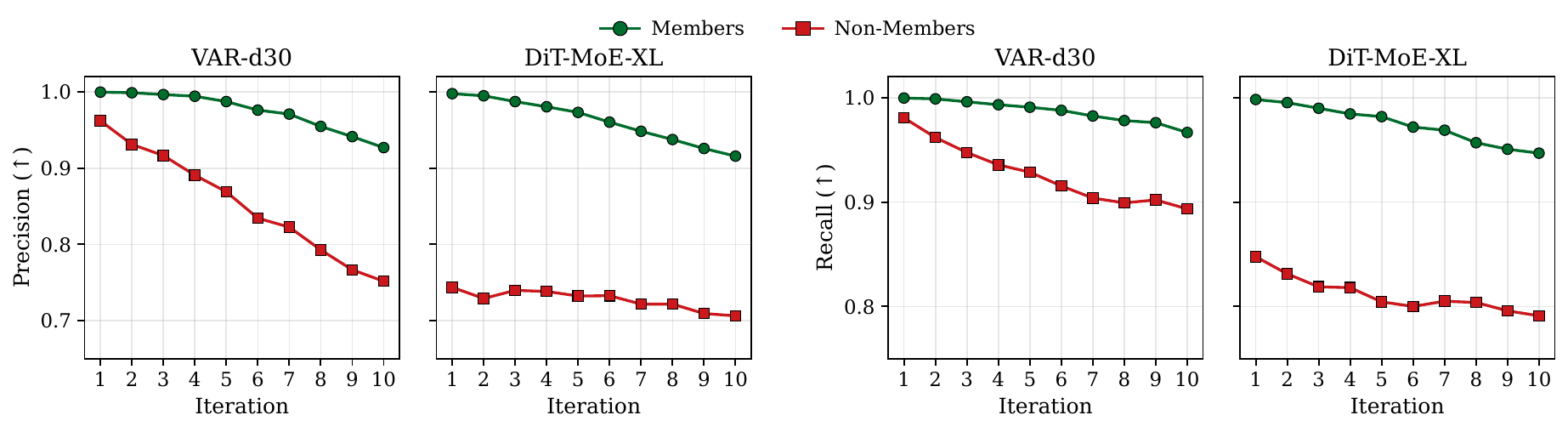}
  \caption{\textbf{Generative Precision and Recall across iterations.} Members stay above non-members in both models and both metrics. For VAR-d30 the gap widens significantly with iteration.}
  \label{fig:prec-rec-models}
\end{figure}

\subsection{The asymmetry persists across model scale}
\label{sec:model-scale}
As illustrated in \Cref{fig:size-ablation}, the membership separation $\Delta\text{FID} = \text{FID}_{\text{nonmem}} - \text{FID}_{\text{mem}}$ stays positive at every iteration for every model we test, so in all of them the gap between members and non-members keeps accumulating along the trajectory. How it depends on scale is family-specific. In VAR the three sizes separate cleanly and the gap grows with the model, sitting near $0.6$ for d20, $1.4$ for d24 and $3.4$ for d30, and it is flat in $t$. In RAR the three sizes track each other closely and all decay over the trajectory, so scale makes little difference. In DiT-MoE the two sizes differ in shape rather than in height alone -- the gap decays for DiT-MoE-XL and grows for DiT-MoE-G -- which is too little to read a trend from. What is robust is that the gap is present and accumulates in every family at every size, not the direction in which it scales.
\begin{figure}[tb!]
  \centering
  \includegraphics[width=\linewidth]{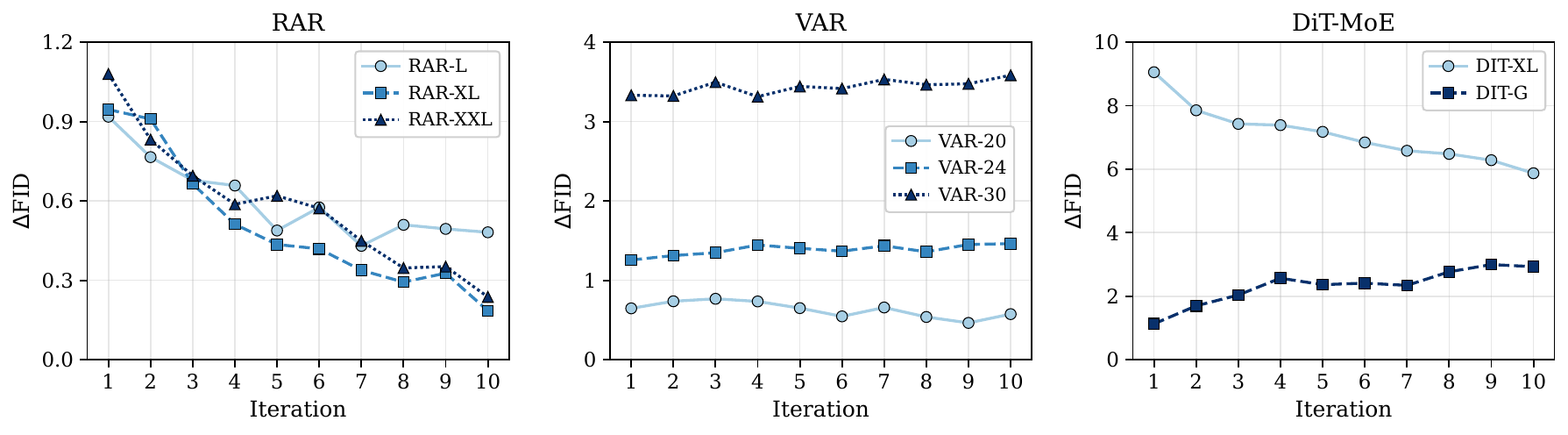}
  \caption{\textbf{Ablation: trajectory asymmetry across model scale.} Membership
  separation $\Delta\text{FID} = \text{FID}_{\text{nonmem}} - \text{FID}_{\text{mem}}$
  over the first ten iterations, for every size in each family. Panels carry
  independent $y$ axes, since the separation differs by an order of magnitude
  between families; darker curves are larger models.}
  \label{fig:size-ablation}
\end{figure}

\section{Trajectory Classifier on Controlled DDPMs}
\label{app:toy-classifier}

On the controlled DDPMs of \Cref{sec:toy} we also read the trajectory directly as a membership signal. A 5-fold cross-validated logistic regression is fit to the $K$-dimensional vector of origin distances $\big[\|x(k)-x(0)\|^{2}/D\big]_{k=1}^{K}$, with $K$ interpolating between the classical \singlequery attack at $K=1$ and a reader of the whole trajectory prefix at $K=15$.
Increasing $K$ never hurts AUC, since the classifier can ignore later coordinates, but at the
boundary of memorization it changes the operating-point metric sharply: for the $n=2000$ model TPR@$5\%$FPR climbs from $0.43$ to $0.99$, while AUC is already above $0.9$ at $K=1$ and merely saturates. An averaged metric therefore hides that a \singlequery attacker would still miss more than half the members at a $5\%$ false-positive budget. The signal is in the shape of the trajectory, not in any single reconstruction error.

\begin{figure*}[hb!]
    \centering
    \begin{subfigure}[t]{0.47\textwidth}
        \centering
        \includegraphics[width=\linewidth]{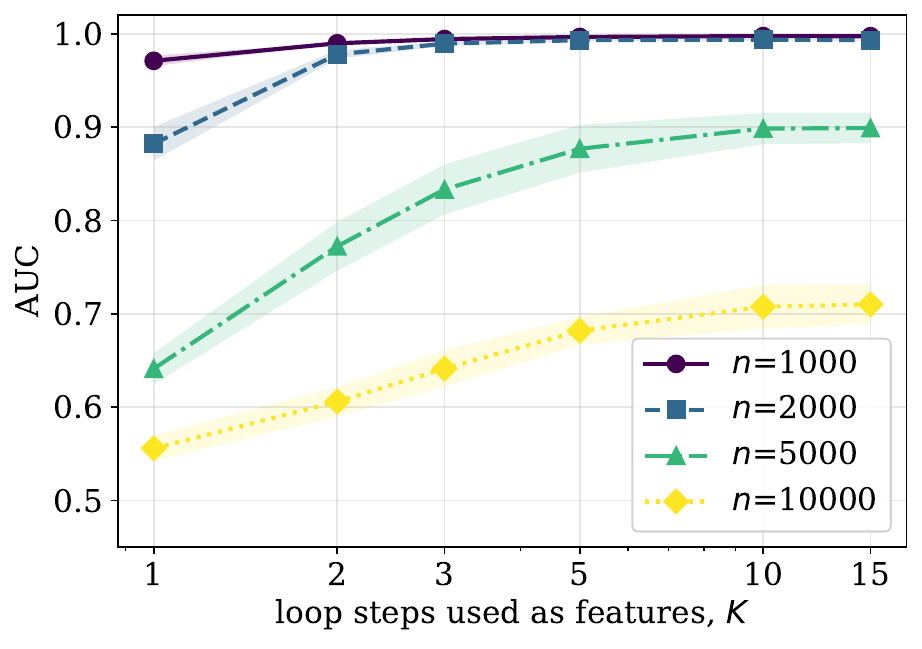}
               \caption{Trajectory-classifier AUC as a function of $K$, the number of trajectory steps used as classifier features. $K=1$ recovers the classical \singlequery MIA; $K=15$ uses the full trajectory prefix.}
        \label{fig:toy_classifier_AUC}
    \end{subfigure}
    \hfill
    \begin{subfigure}[t]{0.47\textwidth}
        \centering
        \includegraphics[width=\linewidth]{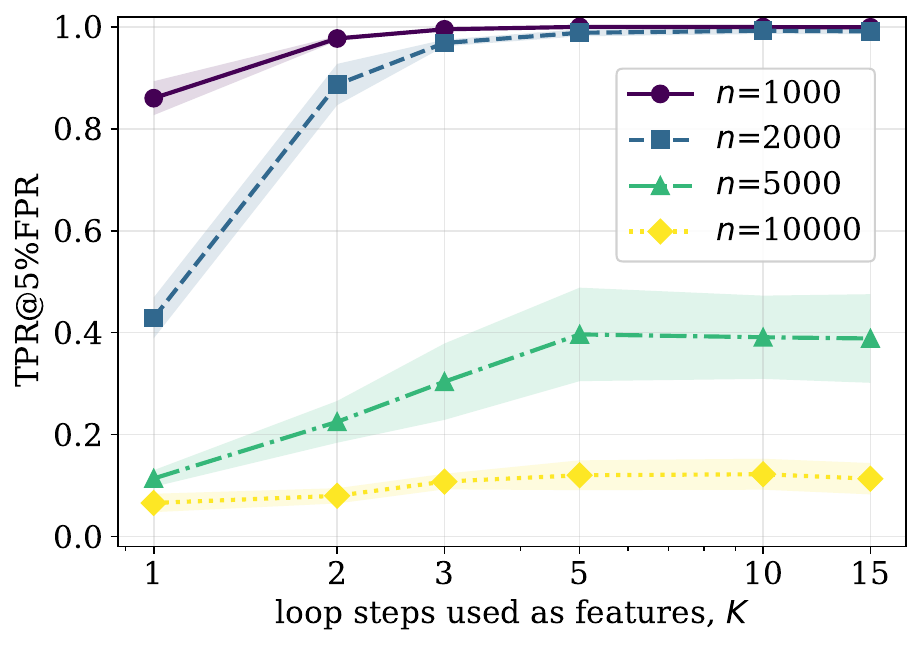}
        \caption{TPR@$5\%$FPR at the same points as
         \subref{fig:toy_classifier_AUC}. Where AUC already
         saturates at $K=1$, the low-FPR regime shows trajectory features lifting the TPR (e.g., training-set size $=2000$: $0.43\!\to\!0.99$).}
        \label{fig:toy_classifier_tpr5}
    \end{subfigure}
    \caption{\textbf{Cumulative trajectory-classifier MIA quality at $\mathrm{denoise\_steps}=2$.}  A 5-fold cross-validated logistic regression is trained on the $K$-dimensional vector of origin-distance features $\big[\|x(k)-x(0)\|^{2}/D\big]_{k=1}^{K}$ per sample, where
     $K$ is the number of loop steps used as features.  For
     weakly memorized models the loop-classifier recovers signal
    that a \singlequery attack ($K=1$) misses, particularly at
    strict operating points.  Bands are $\pm 1$ SD across five
    seeds.}
    \label{fig:toy_classifier}
\end{figure*}
\begin{table}[t]
\centering
\small
\begin{minipage}[t]{0.48\textwidth}
  \centering
  \captionof{table}{\textbf{Generation strength on VAR-d30} (TPR@$1\%$FPR; best per row in bold).}
  \label{tab:korba_var30}
  \begin{tabular}{l|c|ccc}
  \toprule
   & & \multicolumn{3}{c}{+ours} \\
  \cmidrule(lr){3-5}
  Attack & Baseline & $s{=}2$ & $s{=}5$ & $s{=}8$ \\
  \midrule
  Loss & 3.5 & 6.7 & 6.6 & \textbf{6.9} \\
  Hinge & 2.8 & 3.8 & \textbf{5.0} & 3.9 \\
  Min-K\%++ & 9.5 & 14.7 & \textbf{15.5} & 15.4 \\
  SURP & 3.6 & 6.6 & 6.4 & \textbf{6.8} \\
  CAMIA & 15.2 & \textbf{17.2} & 14.8 & 15.0 \\
  \bottomrule
  \end{tabular}
\end{minipage}
\hfill
\begin{minipage}[t]{0.48\textwidth}
  \centering
  \captionof{table}{\textbf{Generation strength on VAR-d24} (TPR@$1\%$FPR; best per row in bold).}
  \label{tab:korba_var24}
  \begin{tabular}{l|c|ccc}
  \toprule
   & & \multicolumn{3}{c}{+ours} \\
  \cmidrule(lr){3-5}
  Attack & Baseline & $s{=}2$ & $s{=}5$ & $s{=}8$ \\
  \midrule
  Loss & 1.8 & 2.2 & \textbf{2.3} & \textbf{2.3} \\
  Hinge & 1.8 & 1.8 & \textbf{2.0} & 1.9 \\
  Min-K\%++ & 2.7 & 3.2 & \textbf{3.4} & 3.3 \\
  SURP & 1.8 & 2.2 & \textbf{2.3} & \textbf{2.3} \\
  CAMIA & 2.3 & 3.4 & 3.9 & \textbf{4.2} \\
  \bottomrule
  \end{tabular}
\end{minipage}

\vspace{1.2em}

\begin{minipage}[t]{0.48\textwidth}
  \centering
  \captionof{table}{\textbf{Generation strength on Pythia-6.9B} (decoding temperature $t$; TPR@$1\%$FPR, MIMIR average; best per row in bold).}
  \label{tab:korba_pythia_6_9b}
  \begin{tabular}{l|c|ccc}
  \toprule
   & & \multicolumn{3}{c}{+ours} \\
  \cmidrule(lr){3-5}
  Attack & Baseline & $t{=}0.0$ & $t{=}0.5$ & $t{=}1.0$ \\
  \midrule
  Loss & 13.2 & \textbf{23.9} & 22.6 & 20.6 \\
  Hinge & 8.9 & \textbf{20.8} & 19.9 & 17.0 \\
  Min-K\%++ & 11.0 & \textbf{23.0} & 22.2 & 18.8 \\
  SURP & 18.4 & \textbf{24.3} & 24.1 & 23.6 \\
  CAMIA & 11.5 & \textbf{23.3} & 21.6 & 19.3 \\
  \bottomrule
  \end{tabular}
\end{minipage}
\hfill
\begin{minipage}[t]{0.48\textwidth}
  \centering
  \captionof{table}{\textbf{Generation strength on Pythia-12B} (decoding temperature $t$; TPR@$1\%$FPR, MIMIR average; best per row in bold).}
  \label{tab:korba_pythia_12b}
  \begin{tabular}{l|c|ccc}
  \toprule
   & & \multicolumn{3}{c}{+ours} \\
  \cmidrule(lr){3-5}
  Attack & Baseline & $t{=}0.0$ & $t{=}0.5$ & $t{=}1.0$ \\
  \midrule
  Loss & 12.9 & 22.7 & \textbf{23.5} & 20.9 \\
  Hinge & 8.7 & \textbf{19.6} & 18.0 & 18.0 \\
  Min-K\%++ & 11.5 & \textbf{22.2} & 20.3 & 20.6 \\
  SURP & 14.6 & 24.0 & \textbf{24.9} & 23.9 \\
  CAMIA & 11.2 & \textbf{21.6} & 20.3 & 20.1 \\
  \bottomrule
  \end{tabular}
\end{minipage}
\end{table}

\clearpage
\section{Theory}
\label{sec:theory}

\textbf{We define memorization as local sharpness of the model's density, and prove that under this definition the trajectory's contraction rate is an estimator of memorization strength.}

Two consequences organize what follows. First, a trajectory is characterized by two numbers, a contraction rate $\lambda$ and a floor $\varepsilon_\sigma/(1-\lambda)$ that it cannot settle below; every trajectory feature we compute is an estimate of one of them, or of a function of both. Second, the same analysis says why spending the query budget on independent perturbations of $Z_0$ instead fails: averaging a base score over a neighborhood introduces a bias that grows with the local curvature, and so penalizes exactly the samples that are most memorized.

We analyze recursive self-generation as a dynamical system on the sample space. The
central claim is that memorization, understood as local sharpness of the model's
own density, makes a training point an attractor of the model's generation map,
with a contraction rate that factorizes into a parameter we control and the
quantity we wish to measure.

\subsection{Setup and assumptions}

One step of recursive self-generation at strength $\sigma$ is
\[
\mathcal{G}_\sigma(z) \;=\; D_\theta(z + \sigma\xi), \qquad \xi\sim\mathcal{N}(0,I),
\]
where $D_\theta$ is the model's denoiser and $z\in\mathbb{R}^d$. Let $p_\theta$
denote the distribution represented by the trained model and
$p^{(\sigma)}_\theta := p_\theta * \mathcal{N}(0,\sigma^2 I)$ its smoothing at
scale $\sigma$. We write $s(y):=\nabla\log p^{(\sigma)}_\theta(y)$ for the score
and $B(x,r) := \{z : \|z-x\|\le r\}$.

\begin{assumption}[Denoiser]
\label{as:denoiser}
$D_\theta$ approximates the optimal denoiser for $p_\theta$ at level $\sigma$, so
that by Tweedie's identity
$\|D_\theta(y) - (y + \sigma^2 s(y))\| \le \delta_\theta$ for all $y$.
\end{assumption}

\begin{assumption}[Local memorization]
\label{as:memorization}
A point $x^\star$ is \emph{memorized at scale $\sigma$ with strength $m>0$} if
there exist $r_0>0$ and $M\ge m$ such that
\[
m I \;\preceq\; -\nabla^2\log p^{(\sigma)}_\theta(z) \;\preceq\; M I
\qquad \forall z\in B(x^\star,r_0),
\qquad\text{and}\qquad
s(x^\star)=0 .
\]
A point is \emph{unmemorized} if no such $m>0$ and $r_0>0$ exist.
\end{assumption}

\begin{assumption}[Regularity]
\label{as:regularity}
$\log p^{(\sigma)}_\theta$ is three times differentiable on $B(x^\star,r_0)$ with
$\|\nabla^3\log p^{(\sigma)}_\theta\|\le L$.
\end{assumption}

Assumption~\ref{as:memorization} is the substantive one, and we do not claim it follows from training; that a model memorizes some samples and not others is the empirical premise of this work. What we claim is that it is \emph{measurable}: the contraction rate it implies is estimable from observed trajectories.

\subsection{One-step contraction}

\begin{proposition}[Contraction at memorized points]
\label{prop:contraction}
Let $x^\star$ satisfy Assumption~\ref{as:memorization} with constants $m,M,r_0$, and let Assumptions~\ref{as:denoiser} and~\ref{as:regularity} hold. Suppose
\[
\sigma^2 (M+m) \le 2
\]
and let $z$ satisfy $\|z-x^\star\|\le r$ with $r + c\sigma\sqrt{d}\le r_0$. Then
\[
\mathbb{E}\big\|\mathcal{G}_\sigma(z)-x^\star\big\|
\;\le\; \lambda\,\|z-x^\star\| \;+\; \varepsilon_\sigma,
\qquad
\lambda \;=\; 1-\sigma^2 m,
\]
with
\[
\varepsilon_\sigma \;=\; \lambda\,\sigma\sqrt{d}
\;+\; \tfrac{1}{2}\sigma^2 L\big(r^2+\sigma^2 d\big)
\;+\; \delta_\theta ,
\]
up to an event of probability $e^{-\Omega(d)}$ on which $z+\sigma\xi$ leaves
$B(x^\star,r_0)$.
\end{proposition}
\begin{proof}[Proof of Proposition~\ref{prop:contraction}]
Write $u := z-x^\star$, $\eta := \sigma\xi$ and $y := x^\star+u+\eta$. By Gaussian
concentration $\|\eta\|\le c\sigma\sqrt{d}$ with probability $1-e^{-\Omega(d)}$;
on this event $y\in B(x^\star,r_0)$ and the expansions below are valid.

Since $s(x^\star)=0$, Taylor's theorem with $H:=\nabla^2\log p^{(\sigma)}_\theta(x^\star)$
and Assumption~\ref{as:regularity} give
\[
s(y)\;=\;H(u+\eta)+R,\qquad \|R\|\le\tfrac{L}{2}\|u+\eta\|^2 .
\]
By Tweedie's formula the optimal denoiser is $D^\star(y)=y+\sigma^2 s(y)$, so by
Assumption~\ref{as:denoiser},
\[
\mathcal{G}_\sigma(z)-x^\star
=\big(I+\sigma^2H\big)(u+\eta)+\sigma^2R+\Delta_\theta,
\qquad \big(\mathbb{E}\|\Delta_\theta\|^2\big)^{1/2}\le\delta_\theta .
\]

\emph{Linear term.} Assumption~\ref{as:memorization} gives $-MI\preceq H\preceq -mI$, so the
eigenvalues of $A:=I+\sigma^2H$ lie in $[1-\sigma^2M,\,1-\sigma^2m]$ and
\[
\|A\|_{\mathrm{op}}\le\max\{|1-\sigma^2m|,\,|1-\sigma^2M|\}=1-\sigma^2m=\lambda ,
\]
where both equalities use the hypothesis $\sigma^2(M+m)\le2$: it is equivalent to
$\sigma^2M-1\le 1-\sigma^2m$, and, since $m\le M$, it also gives $\sigma^2m\le1$, so
$1-\sigma^2m\ge0$ and $\lambda\in[0,1)$. Hence, by the triangle inequality and
$\mathbb{E}\|\eta\|\le\sigma\sqrt{d}$,
\[
\mathbb{E}\|A(u+\eta)\|\;\le\;\|Au\|+\|A\|_{\mathrm{op}}\,\mathbb{E}\|\eta\|
\;\le\;\lambda\|u\|+\lambda\sigma\sqrt{d}.
\]

\emph{Remainder.} $\mathbb{E}\|u+\eta\|^2=\|u\|^2+\sigma^2d\le r^2+\sigma^2d$, so
$\sigma^2\mathbb{E}\|R\|\le\tfrac12\sigma^2L(r^2+\sigma^2d)$. By Jensen,
$\mathbb{E}\|\Delta_\theta\|\le\delta_\theta$. Combining the three bounds gives the
claim.
\end{proof}

\begin{corollary}[Drift floor]
\label{cor:floor}
Under the hypotheses of Proposition~\ref{prop:contraction}, if the iterates remain in $B(x^\star,r)$ then
\[
\mathbb{E}\big\|Z_t-x^\star\big\| \;\le\; \lambda^{t}\,\|Z_0-x^\star\|
\;+\; \varepsilon_\sigma\frac{1-\lambda^{t}}{1-\lambda},
\qquad\text{so}\qquad
\limsup_{t\to\infty}\mathbb{E}\big\|Z_t-x^\star\big\| \;\le\;
\frac{\varepsilon_\sigma}{1-\lambda}\;=\;\frac{\varepsilon_\sigma}{\sigma^2 m}.
\]
This limit is the \emph{drift floor} $\delta_\infty$ of \Cref{sec:measurements}: the trajectory
contracts toward $x^\star$ at rate $\lambda$ but cannot settle closer than the noise and
denoiser error allow.
\end{corollary}
\begin{proof}[Proof of Corollary~\ref{cor:floor}]
Iterate Proposition~\ref{prop:contraction} and sum the geometric series.
\end{proof}
\begin{remark}[Two regimes in $\sigma$]
\label{rem:sigma}
The rate $\lambda=1-\sigma^2m$ improves with $\sigma$, but two constraints bound it
from above: the step-size condition $\sigma^2\le 2/(M+m)$ of
Proposition~\ref{prop:contraction}, beyond which the denoising step overshoots in the sharpest direction and the bound $\lambda$ no longer holds, and the requirement $r+c\sigma\sqrt{d}\le r_0$, which shrinks the admissible starting region as
$\sigma$ grows. The first is tight: at $\sigma^2=2/(M+m)$ the rate bottoms out at
$\lambda=(M-m)/(M+m)$, the condition number of the local curvature. Recursion
therefore has an interior optimum in generation strength, which we observe
empirically for image models in \Cref{sec:strength}.
\end{remark}

\begin{remark}[Unmemorized points]
\label{rem:nonmember}
For a point admitting no $m>0$ in Assumption~\ref{as:memorization}, the same expansion gives
$\|I+\sigma^2H\|_{\mathrm{op}}\ge 1$ and no contraction occurs. The iteration is
then governed by $s$ elsewhere in the space and converges to regions where
$p^{(\sigma)}_\theta$ has mass, independently of the starting point. This is the
drift we observe for non-members in \Cref{sec:drift}.
\end{remark}

\end{document}